\documentclass{article}

\usepackage[main,final]{neurips_2026}
\usepackage{soul}

\usepackage[shortlabels]{enumitem}

\usepackage{microtype}
\usepackage{graphicx}
\usepackage{subcaption}
\usepackage{booktabs} % for professional tables
\usepackage{multirow}

\usepackage{booktabs}
\usepackage{colortbl}
\usepackage{xcolor}
\usepackage{adjustbox}
\usepackage{pifont} % For checkmarks/crosses if needed
\definecolor{rowgray}{gray}{0.92}
\definecolor{winblue}{RGB}{235, 245, 255}
\definecolor{lightgray}{gray}{0.92} % white–grey

\usepackage{hyperref}

\usepackage{amsmath}
\usepackage{amssymb}
\usepackage{mathtools}
\usepackage{amsthm}
\usepackage{amsfonts}
\usepackage{dsfont}
\usepackage{bm}

\usepackage[capitalize,noabbrev]{cleveref}

\theoremstyle{plain}
\newtheorem{theorem}{Theorem}[section]
\newtheorem{proposition}[theorem]{Proposition}

\theoremstyle{definition}

\usepackage{amsthm}
\usepackage{xcolor}

\theoremstyle{remark}

\usepackage[textsize=tiny]{todonotes}

\usepackage{wrapfig}
\usepackage{caption}

\usepackage[utf8]{inputenc} % allow utf-8 input
\usepackage[T1]{fontenc}    % use 8-bit T1 fonts
\usepackage{hyperref}       % hyperlinks
\usepackage{url}            % simple URL typesetting
\usepackage{booktabs}       % professional-quality tables
\usepackage{amsfonts}       % blackboard math symbols
\usepackage{nicefrac}       % compact symbols for 1/2, etc.
\usepackage{microtype}      % microtypography
\usepackage{xcolor}         % colors
\usepackage{wrapfig}
\usepackage{capt-of}

\title{AVIS: Adaptive Test-Time Scaling for Vision–Language Models}

\author{%
  Ahmadreza Jeddi$^{1,2,3\dagger}$\thanks{Equal contribution. Correspondence to: \texttt{ajeddi@cs.toronto.edu}. $\dagger$ Work done at Samsung.}
  \quad Minh N. Le$^{1,2,3\dagger*}$
  \,Amirhossein Kazerouni$^{1,2,3*}$
  \,Hakki Karaimer$^{1}$\\
  \textbf{Hue Nguyen}$^{1}$
  \quad \textbf{Iqbal Mohomed}$^{1}$
  \quad \textbf{Michael Brudno}$^{2,3}$
  \quad \textbf{Konstantinos G. Derpanis}$^{1,3,4}$\\
  \textbf{Alex Levinshtein}$^{1}$
  \quad \textbf{Babak Taati}$^{2,3}$
  \quad \textbf{Radek Grzeszczuk}$^{1\dagger}$\\
  $^1$ AI Center-Toronto, Samsung Electronics \quad
  $^2$ University of Toronto \\
  $^3$ Vector Institute \quad
  $^4$ York University
}

\begin{document}

\maketitle

\begin{abstract}
Modern Vision-Language Models (VLMs) benefit from chain-of-thought prompting and test-time scaling, but these gains often come with prohibitive inference cost due to large visual contexts and long decoding chains. We view this cost through two coupled axes: Visual Context Scaling (VCS), which controls how much visual evidence is passed to the language model, and Visual Reasoning Scaling (VRS), which controls how much inference-time reasoning search is performed. Existing methods typically optimize one axis at a time, leaving the joint allocation of compute across these axes underexplored. We introduce Adaptive Visual Inference Scaling (AVIS), a lightweight policy that adapts both VCS and VRS per query. AVIS realizes VCS through Key Diversity Visual (KDV) pruning, a training-free $O(N)$ key-based rule for removing redundant visual tokens before prefilling, and realizes VRS through adaptive self-consistency, using a learned difficulty predictor to select the number of reasoning rollouts. AVIS is deployment-friendly and compatible with shared-prefill inference, where all rollouts reuse a single prefilling pass and KV cache. Across diverse image and video reasoning benchmarks, AVIS improves the accuracy--compute trade-off relative to VCS-only and VRS-only baselines, and remains effective on top of RL post-trained VLMs while keeping compute and latency low. Project page: \url{https://avis-vlm.github.io/}

\end{abstract}

\addtocontents{toc}{\protect\setcounter{tocdepth}{-1}} % Disable adding to ToC

\definecolor{customred}{HTML}{ED028C}

\section{Introduction}
\label{sec:intro}

\begin{wrapfigure}{R}{0.49\textwidth}
    \centering
    \vspace{-0.5em}
    \includegraphics[width=\linewidth]{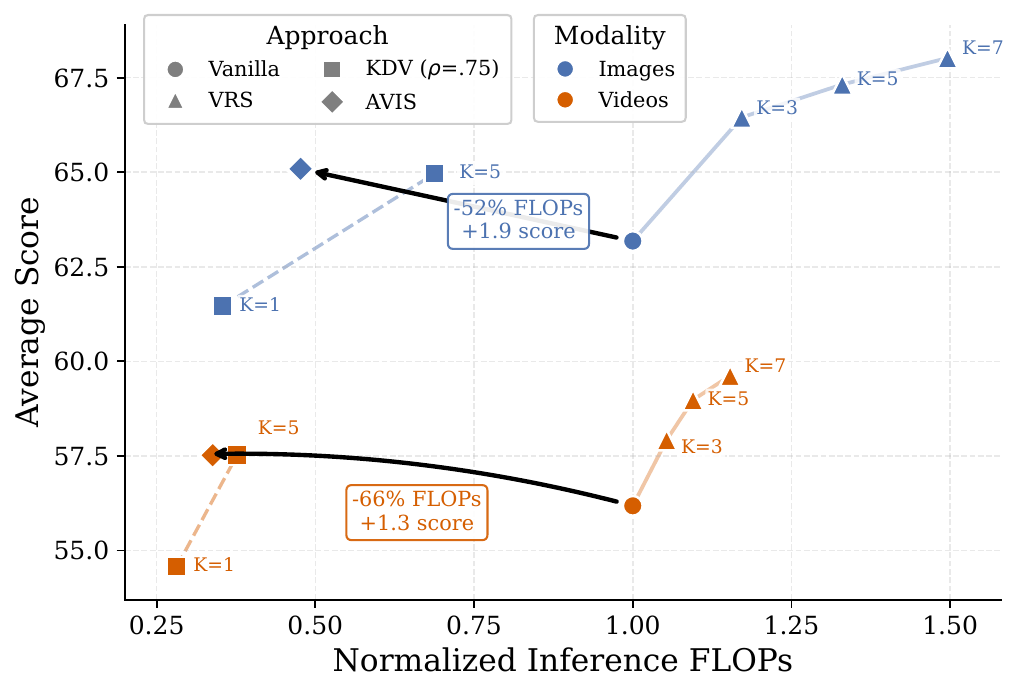}
    \caption{
    \textbf{Accuracy--compute trade-off on Qwen2.5-VL-7B.}
    We report average score over 12 image benchmarks (blue) and 6 video benchmarks (red), with inference FLOPs normalized to Vanilla ($\rho{=}0$, $K{=}1$).
    The plot compares Vanilla, VRS with fixed rollout budgets, KDV with visual pruning, and AVIS with adaptive compute allocation.
    AVIS improves over Vanilla while substantially reducing FLOPs, achieving a better accuracy–compute trade-off by reallocating compute from redundant visual context to additional rollouts only when needed.
    }
    \label{fig:teaser}
    \vspace{-1.0em}
\end{wrapfigure}

Vision-Language Models (VLMs) are rapidly evolving from captioning systems into strong visual reasoners~\cite{bai2025qwen2, li2024llava, zhu2025internvl3}. A key driver has been chain-of-thought (CoT) prompting and reasoning-oriented post-training, which encourage step-by-step rationales instead of single-shot answers~\cite{guo2025deepseek,huang2025vision-r1,wei2022chain}. While effective on complex tasks,  these advances substantially increase inference cost: high-resolution images or videos induce large visual contexts, and CoT-based inference often requires long generations or repeated rollouts, such as Best-of-$N$~\cite{stiennon2020learning,ichihara2025evaluation} and self-consistency~\cite{wang2022self}. As VLMs are increasingly deployed, a central challenge is \emph{adaptive visual reasoning}: \emph{adaptively} allocating inference compute, spending little on easy instances and more on challenging ones, without sacrificing accuracy.

For a given input pair $(I,Q)$, where $I$ is an image or video and $Q$ is a query, inference compute is primarily spent along two coupled axes: \emph{seeing} and \emph{thinking}. The \emph{seeing} compute controls how much visual evidence is passed to the language model as context, such as how many visual tokens are retained after compression or pruning. We refer to this as \textbf{Visual Context Scaling (VCS)}. Increasing VCS preserves richer evidence, but raises prefilling cost and KV-cache memory. In contrast, the \emph{thinking} compute controls inference-time reasoning search, either \emph{sequentially} through longer or iterative CoT traces, or \emph{in parallel} through sampling multiple trajectories and aggregating them with Best-of-$N$ or self-consistency. We refer to this as \textbf{Visual Reasoning Scaling (VRS)}. Increasing VRS expands the search over reasoning paths, but increases decoding compute.

Existing inference methods typically optimize one axis at a time, independently. Visual token pruning and compression~\cite{chen2024image, shang2025llava} reduce VCS for efficiency, often under single-pass decoding, while test-time scaling increases VRS through longer CoT traces or multiple rollouts~\cite{ahmadpour2025limits}, typically treating the visual context as fixed. This leaves a key gap:  \textit{how should a VLM adaptively allocate resources between VCS and VRS for each $(I,Q)$?} We cast this as a compute-allocation problem. Given a multimodal query workload and a configuration $\theta$ that controls both visual context size and reasoning search, we characterize the resulting trade-off surface and show how to exploit it to improve accuracy while minimizing compute.

We present \textbf{Adaptive Visual Inference Scaling (AVIS)}, a lightweight, sample-dependent policy for test-time compute allocation in VLMsf. Given a target workload, AVIS predicts a per-sample inference configuration over both VCS and VRS. Intuitively, it decides how much visual evidence to retain and how much inference-time reasoning search to perform to improve accuracy without unnecessary computation. We formalize this as a compute-allocation problem over the two inference axes introduced above. In AVIS, VCS is implemented through visual token pruning before prefilling, while VRS is implemented through selfconsistency over K chain-of-thought rollouts.

Based on this instantiation, AVIS adopts a two-stage policy for adaptive compute allocation. To reduce VCS, we introduce \textbf{Key Diversity Visual (KDV)} pruning, a training-free visual token selection rule applied before prefilling that uses diversity in attention-key representations as a redundancy signal. KDV runs in $O(N)$ time and requires no additional training. For VRS, AVIS uses self-consistency with majority voting and a difficulty predictor. The predictor allocates the largest rollout counts to an intermediate hard-but-solvable questions, while assigning fewer rollouts to samples at either easy and hard extreme. These choices interact naturally with shared-prefill inference: pruning lowers the one-time prefilling cost, and the resulting KV cache can be reused across all $K$ rollouts, making additional reasoning search substantially cheaper. In this way, AVIS reallocates compute from redundant visual context to useful reasoning search while keeping inference cost low.

We evaluate AVIS across a wide range of visual reasoning tasks, spanning mathematical, visual reasoning, and general VQA benchmarks on both images and videos. We compare against VCS-only and VRS-only baselines, reporting accuracy alongside inference compute measured in FLOPs. \autoref{fig:teaser} shows how our method improves the model performance while also reducing its compute. We further show that shared-prefill makes parallel search efficient: on modern inference engines, AVIS scales reasoning and improves performance while keeping both FLOPs and latency below the corresponding baseline. We also conduct extensive ablations to characterize the VCS--VRS trade-off surface and analyze how AVIS interacts with RL-post-trained models. For consistency, we use Qwen2.5-VL \cite{bai2025qwen2} as our primary backbone due to its popularity and strong open-source ecosystem. Across settings, AVIS’s sample-dependent allocation consistently improves accuracy at fixed or reduced compute, capturing much of the benefit of test-time scaling without its full cost.

Our main contributions are as follows:
\vspace{-0.3cm}
\begin{enumerate}[leftmargin=*,labelindent=0pt,label=\textbf{--}]
    \item We introduce a unified view of VLM inference along two coupled axes, Visual Context Scaling (VCS) and Visual Reasoning Scaling (VRS), and characterize their interaction and trade-offs.
    \vspace{-0.2cm}
    \item We propose AVIS, a lightweight policy that adaptively prunes redundant visual context and scales reasoning search, enabling sample-dependent compute allocation.
    \vspace{-0.2cm}
    \item We introduce KDV, a training-free $O(N)$ visual token pruning rule for VCS, together with a lightweight difficulty predictor for VRS.
\end{enumerate}

\vspace{-1em}
\section{Related Work}
\vspace{-0.75em}
\label{sec:related}
\noindent\textbf{Reasoning VLMs.}
VLMs have improved dramatically, delivering strong results across a broad set of visual and multimodal tasks~\cite{liu2023visual, li2023blip, liu2024llavanext, li2024llava, bai2025qwen2, zhu2025internvl3}. These gains have been enabled by larger and better-aligned LLM backbones, stronger vision encoders with improved resolution/temporal handling, higher-quality data, and reasoning-focused post-training~\cite{huang2025vision-r1, xu2025llava, yao2024mulberry, wang2025vl, jeddi2025puzzle, shen2025vlm, zhang2025r1, jeddi2026does}.
Despite rapid progress in VLM training, \emph{test-time} compute allocation for visual reasoning remains underexplored, especially how to trade off compute spent on visual context versus reasoning-time search.

% \noindent\textbf{Visual Reasoning Scaling (VRS).}
% VRS improves reasoning by spending more inference-time compute on search, either via \emph{parallel} sampling-and-aggregation (e.g., Best-of-$N$, self-consistency and majority vote)~\cite{wang2022self,chen2021evaluating,brown2024large, rakhsha2025majority} or \emph{sequential} refinement (e.g., iterative revise-and-resample, reflection)~\cite{madaan2023self, lightman2023let, wang2024math}.
% This space is well-studied for LLMs, including compute-aware sampling and aggregation policies~\cite{muennighoff2025s1, li2025selfbudgeter, manvi2024adaptive, snell2024scaling-llm, damani2024learning}. For VLMs, some suggest tuning VLMs with RL so that they can more adaptively spend Chain of thought depending on task difficulty~\cite{meng2025cyberv, jin2025cotvid, huang2026learning, chen2025ares}.recent work~\cite{kaya2025efficient, ahmadpour2025limits, zhang2025improve, wu2025aha} reports consistent Best-of-$N$ / majority-vote gains on multimodal benchmarks across model families and post-training recipes. However, these methods typically treat the visual context as fixed, rather than co-optimizing \emph{seeing} (visual context) and \emph{thinking} (search).

\noindent\textbf{Visual Reasoning Scaling (VRS).}
VRS improves reasoning by allocating additional inference-time compute to search, either through \emph{parallel} sampling and aggregation, such as Best-of-$N$, self-consistency, and majority voting~\cite{wang2022self,chen2021evaluating,brown2024large,rakhsha2025majority}, or through \emph{sequential} refinement, such as iterative revision, resampling, and reflection~\cite{madaan2023self,lightman2023let,wang2024math}.
This direction has been extensively studied for LLMs, including compute-aware sampling and aggregation policies~\cite{muennighoff2025s1,li2025selfbudgeter,manvi2024adaptive,snell2024scaling-llm,damani2024learning}.
For VLMs, recent work has explored RL-based tuning to enable models to allocate CoT length more adaptively based on task difficulty~\cite{meng2025cyberv,jin2025cotvid,huang2026learning,chen2025ares}.
Other studies report consistent Best-of-$N$ and majority-vote gains on multimodal benchmarks across model families and post-training recipes~\cite{kaya2025efficient,ahmadpour2025limits,zhang2025improve,wu2025aha}.
However, these approaches typically treat the visual context as fixed, rather than jointly optimizing \emph{seeing} through visual context allocation and \emph{thinking} through reasoning-time search.

\noindent\textbf{Visual Context Scaling (VCS).}
A key efficiency bottleneck in VLM inference is visual context length: high-resolution images, multi-crop inputs, and video frames produce long visual token sequences that inflate prefilling FLOPs and KV-cache memory.
Accordingly, many methods reduce visual tokens/frames before or during prefilling, varying by \emph{where} pruning occurs~\cite{jeddi2025similarity} (ViT-only, early-to-late LLM layers, or hybrid) and by \emph{what signal} is used: (i) \emph{importance/attention-based} heuristics~\cite{chen2024image, xing2024pyramiddrop, zhang2024sparsevlm, yang2025visionzip, huang2025prunevid} or (ii) \emph{redundancy/similarity-based} grouping~\cite{bolya2022token, alvar2025divprune, jeddi2025similarity, zhang2025beyond-attention, zhang2025beyond-text, shang2025llava, yang2025topv, li2025catp}.
Pruning signals interact with deployment constraints: optimized kernels (e.g., FlashAttention~\cite{dao2022flashattention}) do not materialize full attention matrices, so attention-based methods that reconstruct attention at selected layers~\cite{yang2025visionzip, zhang2025vscan} can add substantial overhead and may trigger out-of-memory for long contexts (e.g., videos). Similarity/graph-based approaches can also be costly due to iterative $O(N^2)$ time and memory~\cite{alvar2025divprune,jeddi2025similarity}.
While any of these methods could plug into our pipeline, we opt for a simpler key-diversity proxy inspired by KeyDiff~\cite{park2025keydiff}, which links lower pairwise cosine similarity among keys to higher attention scores and proposes an $O(N)$ eviction rule; we adapt and improve this technique for visual token pruning. Finally, VCS is often studied under single-pass decoding with heuristic prune rates; we extend it to joint, sample-dependent search over both VCS and reasoning-time search.

\begin{figure*}[t]
    \centering
    \resizebox{\textwidth}{!}{%
        \includegraphics{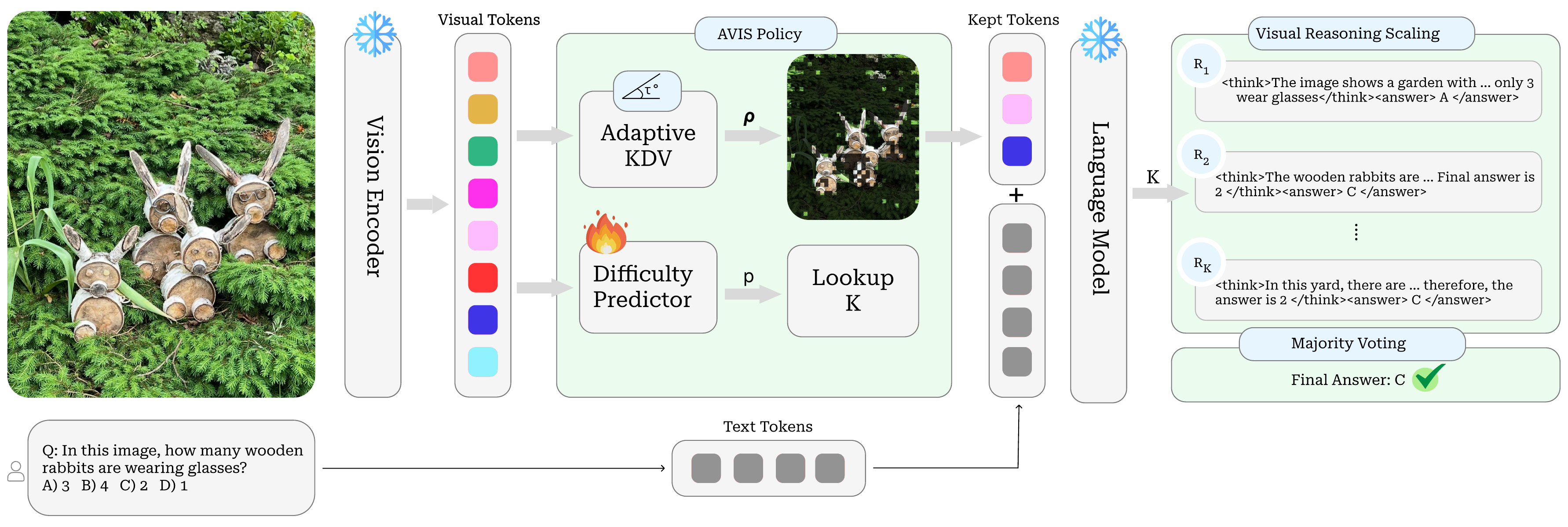}%
    }
    \vspace{-1.5em}
    \caption{\textbf{Adaptive Visual Inference Scaling (AVIS).}
    Given a query $x=(I,Q)$, we first apply an adaptive, training-free key-diversity pruning rule to remove redundant visual tokens and produce a compact visual prefix.
    We then run self-consistency with $K$ shared-prefill rollouts and aggregate answers by majority vote.
    A lightweight difficulty predictor estimates per-sample difficulty and maps it to an appropriate rollout count $K$, enabling adaptive compute allocation across the workload.}
    \label{fig:method_overview}
    \vspace{-1.5em}
\end{figure*}

\vspace{-0.75em}
\section{Problem Setup}
\vspace{-0.5em}
\label{sec:problem_setup}

We consider a multimodal query $x=(I,Q)$, where $I$ is an image or video and $Q$ is a textual prompt. A VLM encodes $I$ into visual embeddings $E_v \in \mathbb{R}^{n_v \times d}$ and tokenizes $Q$ into text embeddings $E_q \in \mathbb{R}^{n_q \times d}$. Conditioned on this multimodal prefix, the language model generates $y \sim P_{\phi}(\cdot \mid I,Q)$, which may include reasoning tokens and a final answer $a(y)$.

At test time, inference compute is spent in three places: (i) the vision encoder forward pass, (ii) a \emph{prefill} pass over the multimodal prefix that builds the KV cache, and (iii) autoregressive decoding. We focus on two sample-dependent inference levers that govern this cost: VCS, which controls how much visual evidence is passed to the language model, and VRS, which controls how much inference-time reasoning search is performed. For each query $x$, we represent an inference configuration by $\theta=(\rho, K)$, where $\rho \in (0,1]$ denotes the pruned fraction of visual tokens and $K$ the number of reasoning trajectories sampled at decoding time. In general, increasing $\rho$ lowers prefilling cost and KV-cache usage, while increasing $K$ raises decoding cost.

\subsection{Adaptive Compute Objective}
\label{sec:adaptive_objective}

Given a target workload $\mathcal{D}_{\text{target}}=\{(x_i,y_i^\star)\}_{i=1}^{D}$, where $x_i=(I_i,Q_i)$, let $B_{\text{base}}$ denote the total inference compute of a baseline VLM evaluated with no visual pruning and $K=1$. Our goal is to improve accuracy by reallocating compute across examples through VCS and VRS, while keeping the total expected compute below this baseline.

For a query $x_i$, let $\hat{a}(x_i;\theta_i)$ denote the prediction under inference configuration $\theta_i=(\rho_i,K_i)$, and let $F(x_i;\theta_i)$ denote the corresponding inference compute. We seek a \emph{sample-dependent} allocation policy $\pi$ that maps each query to a configuration ($\theta_i=\pi(x_i)$) to maximize accuracy while having lower compute than that of the baseline:
\begin{equation}
\max_{\pi}\;\sum_{i=1}^{D}\mathbb{E}\!\left[\mathds{1}\!\left(\hat{a}(x_i;\pi(x_i)) = y_i^\star\right)\right]
\quad
\text{s.t.}\quad
\sum_{i=1}^{D}\mathbb{E}\!\left[F(x_i;\pi(x_i))\right] \le B_{\text{base}} .
\label{eq:budgeted_objective}
\end{equation}
Under mild assumptions on the scaling of prefilling and decoding cost in the shared-prefill regime, the per-example inference compute under configuration $\theta_i=(\rho_i,K_i)$ can be approximated as\footnote{The derivation and empirical validation are deferred to the appendix.}
\begin{equation}
F(x_i;\theta_i) \approx F_v(I_i) + C_{\mathcal{M}} \big(\rho_i n_{v,i} + n_{q,i} + K_i T_i \big),
\label{eq:per_example_cost}
\end{equation}
where $C_{\mathcal{M}}$ is an LLM-dependent constant denoting the average FLOPs incurred by the language model stack per processed token, and $T_i$ is the number of decoded tokens under the vanilla baseline with no pruning and $K=1$. Since $F_v(I_i)$ is fixed across all inference configurations, Eq.~\eqref{eq:per_example_cost} suggests that the total compute allocated to a query can be linearly estimated and controlled through two quantities: the retained prefill tokens and the total decoded tokens.

\subsection{Visual Context Scaling}
\label{sec:vcs}

VCS refers to how much visual evidence is provided to the language model. In general, VCS could be controlled through token compression, pooling, or pruning. In AVIS, we instantiate VCS through visual token pruning before prefilling, i.e., by selecting a subset of visual tokens from $E_v$ and passing only those retained tokens to the language model. This choice directly reduces multimodal prefix length, lowering prefilling FLOPs and KV-cache memory.

To implement this efficiently, we introduce \textbf{Key Diversity Visual (KDV)} pruning, a training-free key-based scoring rule inspired by KeyDiff~\cite{park2025keydiff}, originally proposed for KV-cache eviction in LLMs. KDV scores each visual token using diversity in its attention-key representation, favoring tokens whose keys are less aligned with the average key direction and thus less redundant. To better capture multi-head structure, we compute a head-aware score
\begin{equation}
r_i
=
-\frac{1}{H}
\sum_{h=1}^{H}
\operatorname{CosSim}\!\left(
\mu(\hat{K}_h), \hat{k}_{h,i}
\right),
\label{eq:kdv}
\end{equation}
where $\hat{k}_{h,i}$ is the normalized key of token $i$ at head $h$, and $\mu(\hat{K}_h)$ is the mean normalized key for that head. We then retain the top-ranked tokens according to $r_i$, yielding a retained fraction $\rho$. KDV is training-free, runs in $O(n_v)$ time up to the top-$k$ selection step, and is compatible with optimized attention kernels. KDV prunes visual tokens before the language-model. This placement is particularly well suited to AVIS: the shared prefill can be reused across all reasoning rollouts. Moreover, recent work suggests that pruning before the LLM can outperform in-LLM pruning~\cite{zhang2025beyond-text,zhang2025beyond-attention}, which is consistent with our empirical findings. Together, these choices make KDV a lightweight, practical, and effective VCS mechanism for adaptive multimodal inference.

% We apply KDV after the vision encoder and before the language-model shared prefill. This placement is central to AVIS: it reduces the visual prefix once, before any reasoning rollout, allowing the pruned KV cache to be reused across all parallel samples. It also keeps the method architecture-agnostic, requiring no modification to the vision encoder, language model, or decoding procedure. Moreover, recent work suggests that pruning before the LLM can outperform in-LLM pruning, likely due to better modality alignment~\cite{vispruner,CDPruner}, which is consistent with our empirical findings. These design choices make KDV a lightweight, practical, and effective pruning method for adaptive visual context scaling.

\subsection{Visual Reasoning Scaling}
\label{sec:vrs}

VRS refers to the amount of inference-time reasoning search performed for a query. In principle, this can be increased in several ways, such as longer chain-of-thought traces, iterative refinement, or multiple sampled rollouts. In AVIS, we instantiate VRS through \emph{self-consistency} with majority voting over $K$ independently sampled chain-of-thought trajectories. Larger $K$ expands the search over reasoning paths, but also increases decoding compute.

A key practical advantage of this choice is its compatibility with shared-prefill inference: all $K$ rollouts reuse the same visual features and prefilling KV cache, so increasing $K$ mainly adds decoding-side cost rather than repeating the full multimodal forward pass. In the next section, AVIS uses a lightweight difficulty-aware predictor to choose $K$ adaptively for each query.

Given our adaptive compute allocation objective and the two concrete algorithms that control VCS and VRS, we now introduce \textbf{AVIS}, a lightweight policy that approximately optimizes this objective.

\vspace{-0.5em}
\section{Adaptive Visual Inference Scaling (AVIS)}
\label{sec:policy}
\vspace{-0.25em}
AVIS employs a two-stage approach to solve ~\eqref{eq:budgeted_objective}:
(i) an adaptive visual-token selection rule that reduces redundant visual context, and (ii) a difficulty-aware rule that allocates reasoning rollouts via
self-consistency. The overall architecture is shown in \autoref{fig:method_overview}.
% \vspace{-0.5em}
\subsection{Adaptive Visual Context Scaling}
\label{sec:aakdv}
% \vspace{-0.25em}
\noindent\textbf{Adaptive KDV.}
Recall the KDV retention score $r_i$ \eqref{eq:kdv}, which is the (negative) mean cosine alignment of
token $i$ to the per-head anchor directions. Instead of selecting a fixed top-$N$ set, we obtain an
\emph{adaptive} retained set by thresholding key--anchor angular separation. For a hyperparameter
$\tau\in[0,\pi]$, we retain tokens whose average angle from the anchor exceeds $\tau$, i.e.,
\begin{equation}
\label{eq:aakdv}
S_\tau(x)
\;=\;
\big\{ i \in \{1,\ldots,n_v\} \;:\; r_i \ge -\cos(\tau) \big\}.
\end{equation}
This yields a sample-dependent retained fraction $\rho(x)=|S_\tau(x)|/n_v$ at essentially the same cost
as KDV scoring (linear in $n_v$, up to thresholding). 
% We treat $\tau$ as a hyperparameter; unless stated otherwise we use $\tau=\pi/4$ (45$^\circ$), and ablate $\tau$ in the appendix.

% \vspace{-0.5em}
\subsection{Adaptive Visual Reasoning Scaling}
\label{sec:diff_vrs}
% \vspace{-0.25em}
Self-consistency draws $K$ i.i.d.\
reasoning trajectories and aggregates their extracted answers via majority voting. Let $p(x)$ denote the probability that a \emph{single} trajectory
produces the correct answer. Under the standard i.i.d.\ model,\footnote{For multi-class answers, self-consistency returns the \emph{plurality} winner (largest vote count), not necessarily a strict majority. Eq.~\eqref{eq:psc} corresponds to the binary/``correct vs.\ incorrect'' view and serves as a useful approximation for understanding how accuracy scales with $K$.}
the probability that self-consistency is correct is
\begin{equation}
\label{eq:psc}
P_{\mathrm{SC}}(p,K) \;=\; \sum_{j=\lceil (K+1)/2\rceil}^{K} \binom{K}{j} p^j(1-p)^{K-j}.
\end{equation}
This highlights two regimes: when $p(x)$ is already close to $1$, additional trajectories yield
diminishing returns; when $p(x){<}0.5$, more trajectories may even
degrade accuracy. Our policy, therefore, allocates a larger $K$ for queries that are likely solvable
but not trivial.

\par\noindent\textbf{Learned difficulty predictor.}
We train a lightweight head that predicts, for each query $x$, a \emph{solvability score} $\hat{p}(x)\in[0,1]$ from the visual embeddings $E_v$. The predictor and the rollout-selection rule play distinct roles: the predictor learns an input-dependent signal for whether additional reasoning is likely to help, while the subsequent binning rule is a calibrated decision layer that converts this signal into a discrete rollout budget. This distinction is important because the benefit of self-consistency is not monotonic in difficulty. From~\autoref{eq:psc}, increasing $K$ has little marginal value when $p(x)$ is already high \emph{(very easy)}, since a single trajectory is likely correct, and can also be ineffective when $p(x)$ is very low \emph{(very unlikely to be solvable)}, since majority voting is unlikely to recover the correct answer. The largest expected gain, therefore, occurs for intermediate, \emph{hard-but-solvable} examples. AVIS uses this observation to allocate additional rollouts only to inputs with high predicted marginal utility, rather than simply assigning more compute to all difficult examples.

To obtain supervision, we construct a representative calibration set and run the base VLM with $K_{\max}=10$ rollouts for each example. We estimate its empirical solvability as $p(x_j)=\frac{1}{K_{\max}}\sum_{k=1}^{K_{\max}}\mathds{1}\!\left[\hat y^{(k)}(x_j)=y_j^\star\right],$ and binarize this value with threshold $0.5$ to indicate whether the example is solvable under strong self-consistency. The predictor takes the visual embeddings as input, applies a short stack of 1D Conv--GroupNorm--SiLU layers along the token dimension, followed by global average pooling and a shallow MLP, and outputs logits $s(x_j)\in\mathbb{R}^2$. At test time, we interpret
$\hat p(x_j)=\mathrm{softmax}(s(x_j))_1$ as the predicted solvability score. Finally, because the deployment action space is discrete, $K\in\{1,3,5,7\}$, we map $\hat p(x)$ to a rollout count using a calibrated piecewise-constant budget policy:
\[
\begin{array}{c|cccccc}
\hat p(x) & [0,b_1) & [b_1,b_2) & [b_2,b_3) & [b_3,b_4) & [b_4,b_5) & [b_5,1] \\
\hline
K & 1 & 5 & 7 & 5 & 3 & 1
\end{array}
\]
The bin boundaries $\{b_j\}_{j=1}^5$ are selected once on the calibration set to satisfy the global compute budget and are then fixed for all evaluations. Thus, the binning rule is not intended as a learned model by itself; it is a deployment-friendly action-selection layer over a learned solvability predictor, designed to concentrate reasoning compute on the examples where self-consistency is expected to provide the highest return. The non-monotonic shape follows the inverted-U marginal utility of majority voting~\autoref{eq:psc}: rollouts help most for intermediate $\hat p(x)$ and are wasted at either extreme.

\noindent\textbf{Calibration set.} We construct a calibration set $\mathcal{D}_{\mathrm{cal}}=\{(x_j,y_j^\star)\}_{j=1}^{M}$
representative of the deployment workload. Each example $x_j=(I_j,Q_j)$ is a \emph{multi-choice}
query with a ground-truth answer $y_j^\star$. Following S1 guidelines~\cite{muennighoff2025s1}, we curate $\mathcal{D}_{cal}$ using three principles:
\textbf{Quality}, \textbf{Diversity}, and \textbf{Difficulty}. Starting with a set of 73{,}139 images and videos we arrive at 5000 multi-choice questions ($4000$ images, $1000$ videos) to train the policy. Details of the curation process are provided in \S\autoref{supp:sec:data_curation}. Next, for each $(x_j,y_j^\star)$, we run the base VLM with $K_{\max}{=}10$ rollouts, $p(x_j)$ is then calculated as $\frac{1}{K_{\max}}\sum_{k=1}^{K_{\max}}\mathds{1}\!\big(\hat{y}^{(k)}(x_j)=y_j^\star\big).$ We then binarize this signal using a fixed threshold 0.5 to obtain a supervision label $y_j \;=\; \mathds{1}\!\big(p_j \ge 0.5\big),$ which indicates whether $x_j$ is “solvable” under strong self-consistency.

During the training, the predictor takes as input the sequence of visual token embeddings, processes them with a short stack of 1D Conv–GroupNorm–SiLU layers along the token dimension, followed by global average pooling, and feeds the resulting vector into a shallow MLP producing logits $s(x_j)\in\mathbb{R}^2$.
At test time, we interpret $\hat{p}(x_j) =\mathrm{softmax}\big(s(x_j)\big)_{1}$ as the predicted solvability probability and deterministically map it to a rollout count $K\in\{1,3,5,7\}$ via the binning rule.

\vspace{-0.75em}
\section{Experiments}
\vspace{-0.5em}
\label{sec:experiments}

% \input{tables/main_videos}
% \vspace{-0.25em}
\subsection{Experimental Setup}
We evaluate the effectiveness and efficiency of our adaptive visual reasoning policy, \textsc{AVIS}, on image and video benchmarks.
We use Qwen2.5-VL-7B~\cite{bai2025qwen2} as the backbone and run all methods in CoT mode using the
\texttt{<think>} \texttt{</think>} and \texttt{<answer>} \texttt{</answer>} format~\cite{guo2025deepseek}.
For \textsc{AVIS}, we use the pruning threshold $\tau=\pi/4$, which yields an average pruning ratio of $\approx75\%$; our analysis (~\autoref{fig:vcs_vrs_tradeoff}) shows this choice preserves key visual evidence while providing substantial compute savings.
When $K{=}1$ we use greedy decoding; when $K{>}1$ we sample with temperature $0.7$ and top-$p$ $0.9$ and aggregate by majority vote.
We evaluate image and video benchmarks using VLMEvalKit~\cite{duan2024vlmevalkit}.
% \vspace{-0.25em}
% We evaluate the effectiveness and efficiency of our adaptive visual reasoning policy, \textsc{SLTM}, on image and video benchmarks.
% We use Qwen2.5-VL-7B~\cite{bai2025qwen2} as the backbone and run all methods in CoT mode using the
% \texttt{<think>} \texttt{</think>} and \texttt{<answer>} \texttt{</answer>} format~\cite{guo2025deepseek}.
% For \textsc{AVIS}, we use the default pruning threshold $\tau=\pi/4$, which yields an average pruning ratio of $\approx75\%$; our empirical analysis (~\autoref{fig:vcs_vrs_tradeoff}) shows this choice preserves key visual evidence while providing substantial compute savings.
% When $K{=}1$ we use greedy decoding; when $K{>}1$ we sample with temperature $0.7$ and top-$p$ $0.9$ and aggregate by majority vote.
% We evaluate all image and video benchmarks using VLMEvalKit~\cite{duan2024vlmevalkit}.

\noindent\textbf{Baselines.}
We compare against a \emph{Vanilla} baseline ($K{=}1$, no pruning) and fixed VRS-VCS setups. For pruning baselines, we choose a fixed pruning ratio $\rho{=}0.75$ which is motivated by our finding in \autoref{fig:vcs_vrs_tradeoff} that this pruning level offers a favorable accuracy–efficiency trade-off.
% We also report a fixed joint configuration with $(\rho{=}0.75,\,K{=}5)$.

\noindent\textbf{Image benchmarks.}
We evaluate on a diverse set of multimodal benchmarks. For images our benchmarks span math reasoning (MathVista~\cite{lu2023mathvista}, MathVerse~\cite{zhang2024mathverse}, MathVision~\cite{wang2024measuring}), OCR (DOCVQA~\cite{mathew2021docvqa}), multi-discipline reasoning (MMMU-Pro~\cite{yue2025mmmu}), and general VQA (MME~\cite{fu2025mme}, MMStar~\cite{chen2024we}, MMBench~\cite{liu2024mmbench}, CVBench-2D~\cite{tong2024cambrian}, POPE~\cite{li2023evaluating}, BLINK~\cite{fu2024blink}, TreeBench~\cite{wang2026treebench}).

\noindent\textbf{Video benchmarks.}
For video understanding and reasoning, we evaluate on VideoMME~\cite{fu2025video}, TempCompass~\cite{liu2024tempcompass}, Video-TT~\cite{zhang2025towards}, MVBench~\cite{li2024mvbench}, Q-Bench-Video~\cite{zhang2025q}, and Video-MMMU~\cite{hu2025video}.

\noindent\textbf{Compute measurement.} We compare efficiency using FLOPs normalized by the \emph{Vanilla} setting.
Since all configurations share the same vision encoder, differences are mainly driven by language-model prefill and decoding. We also report wall-clock latency in \S\ref{subsec:latency}.

\begin{table}[t]
\centering
\small
\setlength{\tabcolsep}{3pt}
\caption{
\textbf{Performance across inference strategies.}
We compare VRS-only scaling, KDV-based VCS pruning, fixed joint VCS--VRS configurations, and AVIS, which adaptively selects both $\rho$ and $K$ at test time.
AVIS achieves the best overall compute--performance trade-off across 12 image and 6 video benchmarks.
\#F denotes FLOPs relative to Vanilla ($\rho{=}0,K{=}1$); lower \#F is better.
}
% \vspace{0.5em}
\resizebox{\textwidth}{!}{%
\begin{tabular}{ll|cc|cccccc|cccc|cc}
\toprule
\multirow{3}{*}{\textbf{Category}}
& \multirow{3}{*}{\textbf{Benchmark}}
& \multicolumn{2}{c|}{\textbf{Vanilla}}
& \multicolumn{6}{c|}{\textbf{VRS-Only}}
& \multicolumn{4}{c|}{\textbf{KDV-Prune}}
& \multicolumn{2}{c}{\cellcolor{winblue}\textbf{AVIS}} \\
\cmidrule(lr){3-4}\cmidrule(lr){5-10}\cmidrule(lr){11-14}\cmidrule(lr){15-16}
&
& \multicolumn{2}{c|}{$\rho{=}0,\;K{=}1$}
& \multicolumn{2}{c}{$\rho{=}0,\;K{=}3$}
& \multicolumn{2}{c}{$\rho{=}0,\;K{=}5$}
& \multicolumn{2}{c|}{$\rho{=}0,\;K{=}7$}
& \multicolumn{2}{c}{$\rho{=}75\%,\;K{=}1$}
& \multicolumn{2}{c|}{$\rho{=}75\%,\;K{=}5$}
& \multicolumn{2}{c}{\cellcolor{winblue}Adaptive} \\
\cmidrule(lr){3-4}\cmidrule(lr){5-6}\cmidrule(lr){7-8}\cmidrule(lr){9-10}\cmidrule(lr){11-12}\cmidrule(lr){13-14}\cmidrule(lr){15-16}
&
& Score & \#F
& Score & \#F & Score & \#F & Score & \#F
& Score & \#F & Score & \#F
& \cellcolor{winblue} Score & \cellcolor{winblue} \#F \\
\midrule
\multirow{3}{*}{\textbf{Math}}
& MathVista \cite{lu2023mathvista}    & 67.5 & 1.0 & 68.7 & 1.23 & 70.5 & 1.45 & 70.2 & 1.67 & 66.6 & 0.38 & 68.9 & 0.83 & \cellcolor{winblue}68.1 & \cellcolor{winblue}0.46 \\
& MathVerse \cite{zhang2024mathverse}    & 57.5 & 1.0 & 59.1 & 1.21 & 62.9 & 1.42 & 63.9 & 1.67 & 55.7 & 0.39 & 59.5 & 0.82 & \cellcolor{winblue}58.9 & \cellcolor{winblue}0.45 \\
& MathVision \cite{wang2024measuring}   & 22.4 & 1.0 & 29.9 & 1.35 & 27.0 & 1.71 & 31.6 & 2.04 & 22.3 & 0.45 & 25.4 & 1.04 & \cellcolor{winblue}25.0 & \cellcolor{winblue}0.78 \\
\midrule
\multirow{9}{*}{\textbf{Image VQA}}
& DocVQA \cite{mathew2021docvqa}       & 81.43 & 1.0 & 88.44 & 1.03 & 89.71 & 1.06 & 90.73 & 1.09 & 76.79 & 0.27 & 83.97 & 0.36 & \cellcolor{winblue}87.71 & \cellcolor{winblue}0.32 \\
& MMMU-Pro \cite{yue2025mmmu}    & 42.62 & 1.0 & 43.66 & 1.18 & 45.68 & 1.30 & 46.23 & 1.46 & 40.54 & 0.32 & 43.95 & 0.88 & \cellcolor{winblue}43.10 & \cellcolor{winblue}0.68 \\
& MME \cite{fu2025mme}         & 2263 & 1.0 & 2396 & 1.13 & 2386 & 1.26 & 2397 & 1.39 & 2227 & 0.33 & 2373 & 0.58 & \cellcolor{winblue}2364 & \cellcolor{winblue}0.42 \\
& MMStar \cite{chen2024we}       & 62.0 & 1.0 & 63.6 & 1.20 & 65.8 & 1.40 & 66.4 & 1.61 & 59.0 & 0.38 & 61.5 & 0.79 & \cellcolor{winblue}61.6 & \cellcolor{winblue}0.52 \\
& MMBench \cite{liu2024mmbench}     & 81.4 & 1.0 & 87.8 & 1.16 & 88.3 & 1.32 & 88.6 & 1.48 & 79.2 & 0.36 & 86.8 & 0.68 & \cellcolor{winblue}86.4 & \cellcolor{winblue}0.44 \\
& CVBench \cite{tong2024cambrian}     & 70.6 & 1.0 & 73.9 & 1.11 & 74.9 & 1.22 & 75.5 & 1.42 & 70.5 & 0.33 & 74.3 & 0.55 & \cellcolor{winblue}73.1 & \cellcolor{winblue}0.43 \\
& POPE \cite{li2023evaluating}        & 84.3 & 1.0 & 84.6 & 1.10 & 84.9 & 1.19 & 84.8 & 1.28 & 83.9 & 0.32 & 84.7 & 0.50 & \cellcolor{winblue}84.5 & \cellcolor{winblue}0.38 \\
& BLINK \cite{fu2024blink}        & 56.8 & 1.0 & 56.9 & 1.15 & 58.2 & 1.28 & 57.4 & 1.42 & 51.9 & 0.34 & 55.1 & 0.58 & \cellcolor{winblue}56.7 & \cellcolor{winblue}0.42 \\
& TreeBench \cite{wang2026treebench}   & 37.3 & 1.0 & 41.0 & 1.21 & 40.5 & 1.35 & 41.0 & 1.42 & 38.3 & 0.38 & 36.5 & 0.64 & \cellcolor{winblue}37.5 & \cellcolor{winblue}0.42 \\
\midrule
\multirow{6}{*}{\textbf{Video VQA}}
& Video-MME \cite{fu2025video}   & 65.2 & 1.0 & 66.4 & 1.08 & 70.0 & 1.12 & 70.2 & 1.24 & 62.8 & 0.28 & 66.0 & 0.36 & \cellcolor{winblue}65.8 & \cellcolor{winblue}0.32 \\
& TempCompass \cite{liu2024tempcompass} & 73.6 & 1.0 & 77.8 & 1.02 & 79.0 & 1.04 & 78.6 & 1.06 & 72.4 & 0.26 & 75.6 & 0.30 & \cellcolor{winblue}76.2 & \cellcolor{winblue}0.28 \\
& Video-TT \cite{zhang2025towards}    & 35.2 & 1.0 & 37.6 & 1.02 & 36.8 & 1.04 & 38.6 & 1.05 & 36.4 & 0.27 & 38.4 & 0.30 & \cellcolor{winblue}37.6 & \cellcolor{winblue}0.29 \\
& MVBench \cite{li2024mvbench} & 56.2 & 1.0 & 57.6 & 1.09 & 59.8 & 1.17 & 60.0 & 1.27 & 54.5 & 0.31 & 57.4 & 0.49 & \cellcolor{winblue}57.8 & \cellcolor{winblue}0.46 \\
& Q-Bench-Video \cite{zhang2025q} & 60.5 & 1.0 & 59.9 & 1.06 & 60.5 & 1.12 & 62.2 & 1.18 & 58.4 & 0.29 & 61.3 & 0.42 & \cellcolor{winblue}60.9 & \cellcolor{winblue}0.38 \\
& Video-MMMU \cite{hu2025video}  & 46.4 & 1.0 & 48.2 & 1.05 & 47.8 & 1.08 & 48.1 & 1.12 & 42.9 & 0.28 & 46.4 & 0.33 & \cellcolor{winblue}46.8 & \cellcolor{winblue}0.30 \\
\bottomrule
\end{tabular}
}
\label{tab:main_results}
\vspace{-2em}
\end{table}

\begin{center}
\vspace{-0.5em}
\noindent
\begin{minipage}[t]{0.48\linewidth}
    \centering
    \vspace{0pt}
    \includegraphics[width=\linewidth]{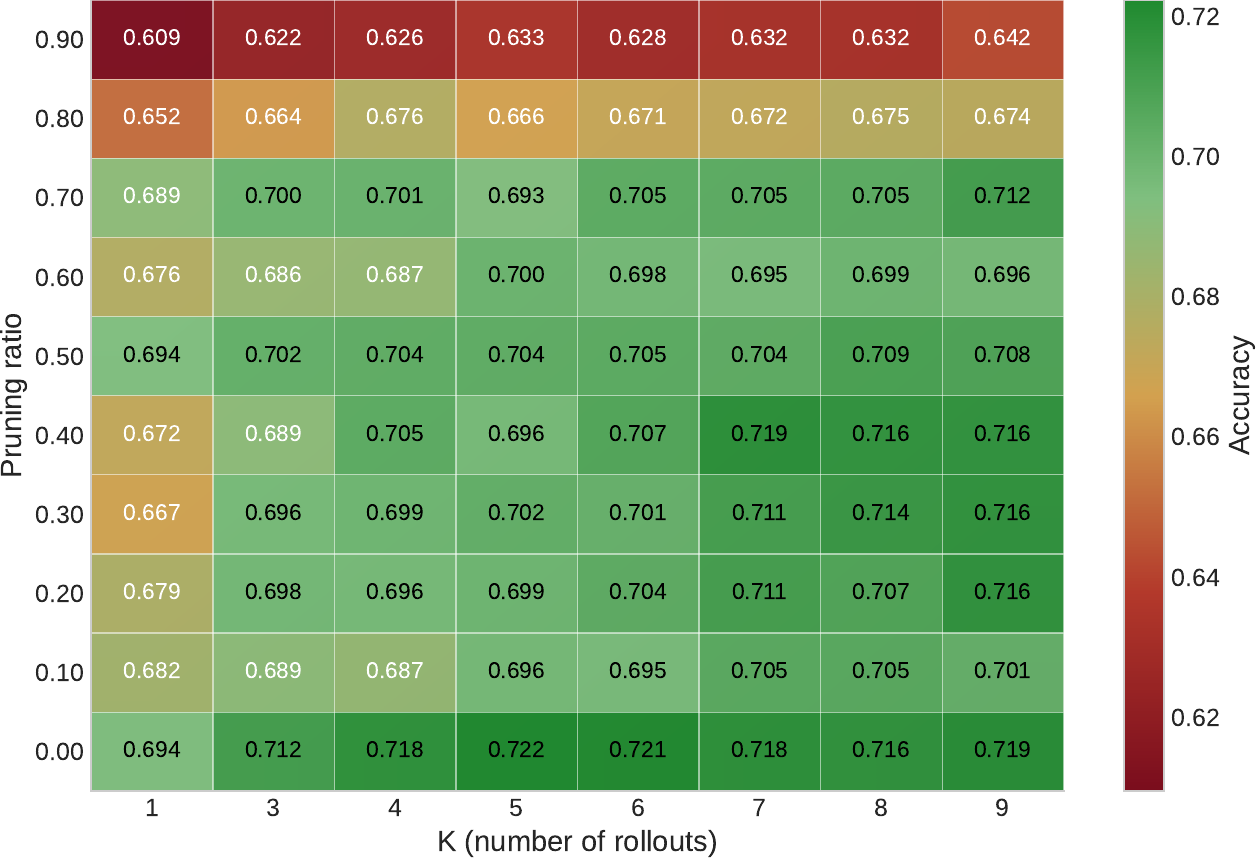}
    \captionof{figure}{\textbf{VCS-VRS Trade-off.} Accuracy heatmap over pruning ratio $\rho$ and rollout count $K$ on the calibration dataset. Each cell shows the resulting accuracy, with lighter colors indicating better performance. Accuracy improves as $K$ increases under varying pruning ratios.}
    \label{fig:vcs_vrs_tradeoff}
\end{minipage}
\hfill
\begin{minipage}[t]{0.48\linewidth}
    \centering
    \vspace{0pt}

\scriptsize
\setlength{\tabcolsep}{1.5pt}
\renewcommand{\arraystretch}{1.3}

% \captionof{table}{%
% \textbf{Performance comparison under matched FLOPs budgets.}
% AVIS uses ${\sim}45\%$ of vanilla FLOPs on average. We select fixed configurations
% ($\rho{=}0.6,\;K{=}1$ and $\rho{=}0.85,\;K{=}5$) that approximate this budget to enable a controlled comparison.
% Under comparable compute, our policy consistently outperforms all fixed baselines. \#F denotes the FLOPs ratio relative to the Vanilla ($\rho{=}0,K{=}1$) baseline.
% }

\captionof{table}{%
\textbf{Matched-FLOPs comparison.}
We compare AVIS against fixed configurations chosen to approximate its average compute budget (${\sim}45\%$ of Vanilla FLOPs).
Under comparable FLOPs, AVIS consistently outperforms fixed VCS--VRS policies, showing the benefit of input-adaptive allocation.
\#F denotes FLOPs relative to Vanilla ($\rho{=}0,K{=}1$).
}

\resizebox{\linewidth}{!}{%
\begin{tabular}{l|cc|cccc|cc}
\toprule
\multirow{3}{*}{\textbf{Benchmark}}
& \multicolumn{2}{c|}{\textbf{Vanilla}}
& \multicolumn{4}{c|}{\textbf{Fixed}}
& \multicolumn{2}{c}{\cellcolor{winblue}\textbf{AVIS}} \\
\cmidrule(lr){2-3}\cmidrule(lr){4-7}\cmidrule(lr){8-9}
&
\multicolumn{2}{c|}{$\rho{=}0,\;K{=}1$}
& \multicolumn{2}{c}{$\rho{=}0.6,\;K{=}1$}
& \multicolumn{2}{c|}{$\rho{=}0.85,\;K{=}5$}
& \multicolumn{2}{c}{\cellcolor{winblue}Adaptive} \\
\cmidrule(lr){2-3}\cmidrule(lr){4-5}\cmidrule(lr){6-7}\cmidrule(lr){8-9}
& Score & \#F
& Score & \#F
& Score & \#F
& \cellcolor{winblue}Score & \cellcolor{winblue}\#F \\
\midrule
MathVista & 67.5 & 1.0 & 66.6 & 0.48 & 65.2 & 0.45 & \cellcolor{winblue}68.1 & \cellcolor{winblue}0.46 \\
MathVerse & 57.5 & 1.0 & 56.5 & 0.44 & 56.4 & 0.48 & \cellcolor{winblue}58.9 & \cellcolor{winblue}0.45 \\
MME & 2263 & 1.0 & 2230 & 0.43 & 2311 & 0.48 & \cellcolor{winblue}2364 & \cellcolor{winblue}0.42 \\
CVBench & 70.6 & 1.0 & 71.2 & 0.40 & 71.4 & 0.47 & \cellcolor{winblue}73.1 & \cellcolor{winblue}0.43 \\
BLINK & 56.8 & 1.0 & 52.1 & 0.45 & 54.5 & 0.48 & \cellcolor{winblue}56.7 & \cellcolor{winblue}0.42 \\
TreeBench & 37.3 & 1.0 & 36.8 & 0.39 & 35.1 & 0.43 & \cellcolor{winblue}37.5 & \cellcolor{winblue}0.42 \\
\bottomrule
\end{tabular}
}

\vspace{0.4em}

\label{tab:iso_flop}

\end{minipage}
\vspace{-0.5em}
\end{center} 

\subsection{Main Results}
\vspace{-0.5em}
\label{subsec:main_results}

\autoref{tab:main_results} compares AVIS against a broad set of inference policies, including the \emph{Vanilla} baseline, VRS-only and VCS-only variants, and fixed joint VCS--VRS configurations.
\autoref{fig:teaser} summarizes the corresponding Pareto trade-off between performance and compute, while \autoref{tab:iso_flop} provides an iso-compute comparison across policies.

\noindent\textbf{Findings.} Tables~\ref{tab:main_results} and \ref{tab:iso_flop} highlight three main trends:
\begin{itemize}[leftmargin=*,labelindent=0pt,label=\textbf{--}]
    \item As expected, increasing VRS improves accuracy at the cost of higher compute, whereas aggressive VCS reduction degrades performance when applied in isolation. Coupling visual pruning with reasoning scaling yields a better operating point: even fixed joint configurations can outperform the base model while using fewer FLOPs than the baseline.

    \item AVIS consistently provides a stronger compute--performance trade-off. Compared to the \emph{Vanilla} baseline, it improves mean accuracy on image and video benchmarks by 3\% and 2.4\%, respectively, while reducing compute by 52\% and 66\%. Compared to the closest fixed baseline $(\rho{=}0.75, K{=}5)$, AVIS achieves higher accuracy with 40\% lower compute.

    \item Under approximately matched FLOPs, AVIS outperforms the closest competing policy by 3.7\%.
\end{itemize}

\vspace{-0.5em}
\subsection{Latency}
\vspace{-0.5em}
\label{subsec:latency}
\begin{wraptable}{R}{0.4\textwidth}
\centering
\vspace{-1.3em}
\scriptsize
\setlength{\tabcolsep}{4pt}
\renewcommand{\arraystretch}{0.95}

\caption{\textbf{Latency Comparison.} Decoding latency for 100 image-question pairs with shared prefilling. Measurements are obtained on a single NVIDIA L40 (48\,GB) GPU using vLLM v0.10.0~\cite{kwon2023efficient}.}
\vspace{-0.6em}

\resizebox{\linewidth}{!}{%
\begin{tabular}{l|c c|c|c}
\toprule
\multirow{2}{*}{\textbf{Method}} 
& \multicolumn{2}{c|}{\textbf{Setting}} 
& \multirow{2}{*}{\begin{tabular}[c]{@{}c@{}}\textbf{Total time}\\\textbf{(sec.)}\end{tabular}} 
& \multirow{2}{*}{\begin{tabular}[c]{@{}c@{}}\textbf{Reduction}\\\textbf{ratio}\end{tabular}} \\
\cmidrule(lr){2-3}
& $\boldsymbol{\rho}$ & $\boldsymbol{K}$ & & \\
\midrule
\textbf{Vanilla} 
& 0 & 1 & 795 & 1.0 \\
\midrule
% \multirow{3}{*}{\textbf{VRS-Only}} 
% & 0 & 3 & 760 & 1.05 \\
% & 0 & 5 & 810 & 0.98 \\
% & 0 & 7 & 1280 & 0.62 \\

% \multirow{2}{*}{\textbf{KDV-Prune}} 
% & 75\% & 1 & 558 & 1.42 \\
% & 75\% & 5 & 789 & 1.01 \\
\multirow{3}{*}{\textbf{Fixed}} 
& 0\% & 5 & 810 & 0.98 \\
& 75\% & 1 & 558 & 1.42 \\
& 75\% & 5 & 789 & 1.01 \\
\midrule
\textbf{\cellcolor{winblue}AVIS} 
& \multicolumn{2}{c|}{\cellcolor{winblue}Adaptive}
& \cellcolor{winblue}757 & \cellcolor{winblue}1.05 \\

\bottomrule
\end{tabular}
}
\label{tab:latency}
\vspace{-1.9em}
\end{wraptable}

We also measure wall-clock latency to validate the benefits of combining visual context reduction with self-consistency.
Table~\ref{tab:latency} reports decoding time for 100 image--question pairs under shared prefilling.
The Vanilla baseline ($\rho{=}0$, $K{=}1$) takes \textbf{795\,s}.
Among fixed configurations, increasing rollouts ($\rho{=}0\%$, $K{=}5$) slightly exceeds the baseline at \textbf{810\,s}, while KDV pruning ($\rho{=}75\%$, $K{=}1$) reduces latency to \textbf{558\,s}.
Combining both axes ($\rho{=}75\%$, $K{=}5$) brings latency back near baseline at \textbf{789\,s}, while achieving higher accuracy with fewer visual tokens.
AVIS achieves \textbf{757\,s}, confirming that adaptive VCS--VRS yields tangible latency benefits in shared-prefill setups.

% \subsection{Latency}
% \label{subsec:latency}
% \input{tables/latency}

% We also measure wall-clock latency to validate the benefits of combining visual context reduction with self-consistency.
% Table~\ref{tab:latency} reports the decoding time for 100 image--question pairs under shared prefilling.
% The Vanilla baseline ($\rho{=}0$, $K{=}1$) takes \textbf{795\,s} (reduction ratio $1.0$).
% VCS-only with KDV pruning at $\rho{=}75\%$ and $K{=}1$ reduces latency to \textbf{558\,s} (ratio $1.42$), corresponding to roughly a 30\% decrease in wall-clock time.
% The joint configuration $(\rho{=}75\%, K{=}5)$ matches the baseline latency while using fewer visual tokens and achieving higher accuracy. This confirms that coupling VCS with VRS indeed yields tangible latency benefits in shared-prefill setups.

\vspace{-0.5em}
\subsection{RL Post-trained VLMs}
\label{subsec:rl_models}
\vspace{-0.25em}
Recent VLM work improves multimodal reasoning through GRPO-style RL post-training~\cite{guo2025deepseek, huang2025vision-r1,wang2025vl}, reporting gains across image and video reasoning benchmarks.
In \autoref{tab:rl_results}, we evaluate three well-regarded RL post-trained models under a fixed VCS--VRS setting of $(\rho{=}75\%,K{=}5)$ as well as AVIS, and include the baseline Qwen2.5-VL results for reference.
Across models and benchmarks, AVIS preserves or improves accuracy while substantially reducing FLOPs.
Overall, these results show that the benefits from our approach extend well to RL post-trained VLMs.

\begin{table*}[t]
\centering
\small
\setlength{\tabcolsep}{4pt}
\caption{
\textbf{AVIS on RL post-trained VLMs.}
We compare each checkpoint under Vanilla, fixed joint scaling $(\rho{=}75\%,K{=}5)$, and AVIS with adaptive $\rho$ and $K$.
Across 6 image benchmarks, AVIS preserves or improves accuracy with substantially lower FLOPs, showing that adaptive VCS--VRS allocation transfers to RL post-trained VLMs.
}
\resizebox{\textwidth}{!}{%
\begin{tabular}{ll|>{\centering\arraybackslash}p{1cm}>{\centering\arraybackslash}p{1cm}|cc cc cc cc cc cc}
\toprule
\multirow{2}{*}{\textbf{Checkpoint}}
& \multirow{2}{*}{\textbf{Method}} 
& \multicolumn{2}{c|}{\textbf{Setting}}
& \multicolumn{2}{c}{MathVista}
& \multicolumn{2}{c}{MathVision}
& \multicolumn{2}{c}{MME}
& \multicolumn{2}{c}{MMStar}
& \multicolumn{2}{c}{CVBench}
& \multicolumn{2}{c}{POPE} \\
& 
& $\boldsymbol{\rho}$ & $\boldsymbol{K}$ 
& Score & \#F
& Score & \#F
& Score & \#F
& Score & \#F
& Score & \#F
& Score & \#F\\
\midrule
\multirow{4}{*}{\textbf{Qwen 2.5 VL~\cite{bai2025qwen2}}}
& \textbf{Baseline}
& $0\%$ & 1 
& 67.5 & 1.0 & 22.4 & 1.0 
& 2263 & 1.0 & 62.0 & 1.0 & 70.6 & 1.0 & 84.3 & 1.0 \\
\cmidrule(lr){2-16}
& \textbf{Fixed}
& $75\%$ & 5
& 68.9 & 0.83 & 25.4 & 1.04 
& 2373 & 0.58 & 61.5 & 0.79 & 74.3 & 0.55 & 84.7 & 0.50 \\
\cmidrule(lr){2-16}
& \cellcolor{winblue}\textbf{AVIS}
& \multicolumn{2}{c|}{\cellcolor{winblue}Adaptive}
& \cellcolor{winblue}68.1 & \cellcolor{winblue}0.46 & \cellcolor{winblue}25.0 & \cellcolor{winblue}0.78 & \cellcolor{winblue}2364 & \cellcolor{winblue}0.42 & \cellcolor{winblue}61.6 & \cellcolor{winblue}0.52 & \cellcolor{winblue}73.1 & \cellcolor{winblue}0.43 & \cellcolor{winblue}84.5 & \cellcolor{winblue}0.38 \\
\midrule
\multirow{4}{*}{\textbf{VL-Rethinker~\cite{wang2025vl}}}
& \textbf{Baseline}
& $0\%$ & 1 
& 72.2 & 1.0 & 33.2 & 1.0 
& 2328 & 1.0 & 64.6 & 1.0 & 76.9 & 1.0 & 82.9 & 1.0 \\
\cmidrule(lr){2-16}
& \textbf{Fixed}
& $75\%$ & 5
& 69.9 & 0.88 & 36.5 & 1.25 & 2402 & 0.61 & 59.2 & 0.84 & 73.6 & 0.64 & 82.0 & 0.52 \\
\cmidrule(lr){2-16}
& \cellcolor{winblue}\textbf{AVIS}
& \multicolumn{2}{c|}{\cellcolor{winblue}Adaptive}
& \cellcolor{winblue}71.8 & \cellcolor{winblue}0.43 & \cellcolor{winblue}37.3 & \cellcolor{winblue}0.82 & \cellcolor{winblue}2463 & \cellcolor{winblue}0.46 & \cellcolor{winblue}61.3 & \cellcolor{winblue}0.54 & \cellcolor{winblue}76.1 & \cellcolor{winblue}0.42 & \cellcolor{winblue}84.7 & \cellcolor{winblue}0.40 \\
\midrule
\multirow{4}{*}{\textbf{Vision-R1~\cite{huang2025vision-r1}}}
& \textbf{Baseline}
& $0\%$ & 1 
& 68.5 & 1.0 & 39.5 & 1.0 
& 2408 & 1.0 & 63.4 & 1.0 & 72.7 & 1.0 & 88.5 & 1.0 \\
\cmidrule(lr){2-16}
& \textbf{Fixed}
& $75\%$ & 5
& 69.3 & 0.81 & 40.8 & 1.20 
& 2436 & 0.63 & 61.5 & 0.75 & 73.4 & 0.56 & 88.9 & 0.57 \\
\cmidrule(lr){2-16}
& \cellcolor{winblue}\textbf{AVIS}
& \multicolumn{2}{c|}{\cellcolor{winblue}Adaptive}
& \cellcolor{winblue}69.6 & \cellcolor{winblue}0.39 & \cellcolor{winblue}41.1 & \cellcolor{winblue}0.78 & \cellcolor{winblue}2469 & \cellcolor{winblue}0.43 & \cellcolor{winblue}65.1 & \cellcolor{winblue}0.49 & \cellcolor{winblue}74.4 & \cellcolor{winblue}0.43 & \cellcolor{winblue}89.1 & \cellcolor{winblue}0.40 \\
\midrule
\multirow{4}{*}{\textbf{OpenVLThinker~\cite{deng2025openvl}}}
& \textbf{Baseline}
& $0\%$ & 1 
& 68.7 & 1.0 & 27.9 & 1.0 
& 2291 & 1.0 & 65.0 & 1.0 & 72.9 & 1.0 & 82.6 & 1.0 \\
\cmidrule(lr){2-16}
& \textbf{Fixed}
& $75\%$ & 5
& 69.8 & 0.79 & 28.6 & 1.05 
& 2435 & 0.55 & 63.4 & 0.72 & 74.2 & 0.50 & 83.2 & 0.48 \\
\cmidrule(lr){2-16}
& \cellcolor{winblue}\textbf{AVIS}
& \multicolumn{2}{c|}{\cellcolor{winblue}Adaptive}
& \cellcolor{winblue}70.6 & \cellcolor{winblue}0.41 & \cellcolor{winblue}28.4 & \cellcolor{winblue}0.77 & \cellcolor{winblue}2437 & \cellcolor{winblue}0.39 & \cellcolor{winblue}64.2 & \cellcolor{winblue}0.48 & \cellcolor{winblue}74.8 & \cellcolor{winblue}0.39 & \cellcolor{winblue}84.5 & \cellcolor{winblue}0.34 \\
\bottomrule
\end{tabular}
}
\label{tab:rl_results}
\vspace{-1em}
\end{table*}

\vspace{-0.5em}
\subsection{Ablation Studies}
\label{subsec:ablations}
\vspace{-0.25em}

\noindent\textbf{Ablating KDV.}
KDV adapts the retained visual context to each input, preserving tokens around salient regions while removing large, low-information background areas. This behavior supports the premise that substantial visual pruning is possible without discarding evidence needed for downstream reasoning. Quantitatively, KDV also compares favorably against other training-free pruning baselines at a fixed $75\%$ pruning rate, achieving the best average performance and retaining $97.99\%$ of the \emph{Vanilla} accuracy across eight benchmarks. We provide qualitative visualizations of adaptive pruning in \autoref{fig:adaptive_prune} and the full pruning comparison in Appendix~\ref{supp:sec:kdv_pruning}.

% \paragraph{Ablating the difficulty predictor.}
% We analyze the learned VRS controller to verify that AVIS does not simply increase reasoning compute uniformly. As shown in \autoref{fig:judged_k_barplots}, the selected rollout distribution is strongly dataset-dependent: the average rollout count is $2.12$ on MMBench, $2.53$ on MME, and $3.79$ on Video-TT. This indicates that the predictor learns non-trivial compute allocation patterns, assigning $K=1$ to many easy or low-gain examples while reserving larger budgets for samples where additional self-consistency is more likely to help. We further consider various alternative difficulty predictor architectures in \autoref{tab:predictor_arch_ablation}. Compared with Perceiver-~\cite{jaegle2021perceiver} and Transformer-style predictors, our predictor achieves the highest held-out prediction accuracy ($79.2\%$) and uses substantially fewer rollouts at comparable downstream accuracy. For example, on MMBench, it reduces the average rollout count from $3.25$ to $2.12$ relative to the Perceiver predictor, lowering FLOPs from $0.59\times$ to $0.44\times$ while maintaining similar accuracy. These results support the role of the learned difficulty predictor as a compute-allocation mechanism: sharper solvability estimates allow AVIS to avoid unnecessary rollouts while preserving most of the accuracy gains from self-consistency. We provide the full rollout-allocation analysis and predictor variation results in Appendix~\ref{app:adaptive_rollout_allocation} and~\ref{app:difficulty_predictor_ablation}, respectively.

\paragraph{Ablating the difficulty predictor.}
We analyze the learned VRS controller to verify that AVIS does not simply increase reasoning compute uniformly. As shown in \autoref{fig:judged_k_barplots} in the appendix, the selected rollout distribution is strongly dataset-dependent, with average rollout counts of $2.12$ on MMBench, $2.53$ on MME, and $3.79$ on Video-TT. This suggests that the predictor learns non-trivial allocation patterns, assigning $K=1$ to many easy or low-gain examples while reserving larger budgets where self-consistency is more likely to help. We also compare alternative difficulty predictor architectures in \autoref{tab:predictor_arch_ablation}. Compared with Perceiver~\cite{jaegle2021perceiver} and Transformer-style predictors, our predictor achieves the highest held-out accuracy ($79.2\%$) and uses fewer rollouts at comparable downstream accuracy. For example, on MMBench, it reduces the average rollout count from $3.25$ to $2.12$ relative to the Perceiver predictor, lowering FLOPs from $0.59\times$ to $0.44\times$ while maintaining similar accuracy. These results support the difficulty predictor as an effective compute-allocation mechanism: sharper solvability estimates let AVIS avoid unnecessary rollouts while preserving most self-consistency gains. Full rollout-allocation and predictor-architecture results are provided in Appendix~\ref{app:adaptive_rollout_allocation} and~\ref{app:difficulty_predictor_ablation}.

\noindent\textbf{Adaptive vs.\ fixed allocation per axis.}
\begin{table*}[t]
\centering
\small
\setlength{\tabcolsep}{3pt}
\caption{
\textbf{Adaptive vs.\ fixed allocation per axis.}
We compare one-axis adaptive policies against AVIS, which jointly adapts visual context $\rho$ and reasoning budget $K$. 
Joint adaptation achieves the strongest compute--performance trade-off, showing that VCS and VRS are complementary.
}
\setlength{\tabcolsep}{12pt}
\resizebox{\textwidth}{!}{%
\begin{tabular}{l|cc|cc|cccc|cc}
\toprule
\multirow{3}{*}{\textbf{Benchmark}}
& \multicolumn{2}{c|}{\textbf{Vanilla}}
& \multicolumn{2}{c|}{\textbf{Fixed}}
& \multicolumn{4}{c|}{\textbf{One-Axis Adaptive}}
& \multicolumn{2}{c}{\cellcolor{winblue}\textbf{AVIS}} \\
\cmidrule(lr){2-3}\cmidrule(lr){4-5}\cmidrule(lr){6-9}\cmidrule(lr){10-11}
&
\multicolumn{2}{c|}{$\rho{=}0,\;K{=}1$}
& \multicolumn{2}{c|}{$\rho{=}75\%,\;K{=}5$}
& \multicolumn{2}{c}{Adapt-$\rho$, $K{=}5$}
& \multicolumn{2}{c|}{$\rho{=}75\%$, Adapt-$K$}
& \multicolumn{2}{c}{\cellcolor{winblue}Adaptive} \\
\cmidrule(lr){2-3}\cmidrule(lr){4-5}\cmidrule(lr){6-7}\cmidrule(lr){8-9}\cmidrule(lr){10-11}
& Score & \#F
& Score & \#F
& Score & \#F
& Score & \#F
& \cellcolor{winblue}Score & \cellcolor{winblue}\#F \\
\midrule
MathVision & 22.4 & 1.0 & 22.3 & 1.04 & 24.9 & 1.05 & 25.3 & 0.83 & \cellcolor{winblue}25.0 & \cellcolor{winblue}0.78 \\
MME        & 2263 & 1.0 & 2373 & 0.58 & 2347 & 0.52 & 2280 & 0.41 & \cellcolor{winblue}2364 & \cellcolor{winblue}0.42 \\
DocVQA     & 81.4 & 1.0 & 83.97 & 0.36 & 84.7 & 0.35 & 86.2 & 0.31 & \cellcolor{winblue}87.7 & \cellcolor{winblue}0.32 \\
CVBench    & 70.6 & 1.0 & 74.3 & 0.55 & 73.8 & 0.53 & 73.2 & 0.49 & \cellcolor{winblue}73.1 & \cellcolor{winblue}0.43 \\
POPE       & 84.3 & 1.0 & 84.7 & 0.50 & 84.8 & 0.49 & 84.1 & 0.41 & \cellcolor{winblue}84.5 & \cellcolor{winblue}0.38 \\
\bottomrule
\end{tabular}
}
\label{tab:adaptive_ablation}
\vspace{-1em}
\end{table*}
We also evaluate mixed policies that adapt only one axis while keeping the other fixed, as shown in \autoref{tab:adaptive_ablation}. 
These ablations isolate the effects of adaptive visual context allocation and adaptive reasoning allocation. 
Adaptive $\rho$ improves over fixed KDV pruning on several benchmarks, showing that input-dependent pruning retains more useful visual evidence than a global pruning ratio. 
Adaptive $K$ similarly improves the compute--performance trade-off by concentrating rollouts on examples where self-consistency is most beneficial. 
Combining the two yields the strongest overall policy: on DocVQA, AVIS reaches $87.7$, exceeding both Adapt-$\rho$, $K{=}5$ ($84.7$) and $\rho{=}75\%$, Adapt-$K$ ($86.2$), with comparable or lower compute. 
These results show that visual context and reasoning budget are best adapted jointly rather than independently.

\noindent\textbf{VCS--VRS trade-off.}
In \autoref{fig:vcs_vrs_tradeoff}, we sweep the pruning ratio $\rho$ from $0$ to $0.9$ and the rollout count $K$ from $1$ to $9$ (except $K=2$) on the calibration set and plot the resulting accuracy.
Along each row (fixed $\rho$), accuracy generally increases as $K$ grows.
Along each column (fixed $K$), moderate pruning ($\rho \approx 0.2$--$0.7$) keeps the model competitive with, or slightly better than the no-pruning case, indicating the high redundancy and the effectiveness of our KDV.
Only in the extreme regime ($\rho{\ge}0.8$) do we observe clear accuracy degradation, showing that KDV can safely remove a large fraction of visual tokens before performance begins to drop.

\vspace{-0.75em}
\section{Conclusion and discussion}
\vspace{-0.75em}
\label{sec:discussion}
We studied adaptive visual reasoning in VLM inference through the joint allocation of visual context and reasoning compute. We introduced AVIS, a deployment-friendly policy that prunes visual context with lightweight KDV and reallocates the saved compute to parallel self-consistency via a difficulty-aware rollout selector. Across image and video benchmarks, AVIS improves the accuracy--compute trade-off over fixed configurations, showing that visual context and reasoning effort should be treated as coupled inference-time resources rather than independent design choices.

% AVIS is a first instantiation of adaptive VCS--VRS allocation. Predicting per-sample difficulty is itself a hard problem that we do not claim to solve; instead, following recent adaptive-compute work for LLMs~\cite{snell2024scaling-llm,damani2024learning}, we approximate it with a coarse binary solvability signal and a lightweight head over visual embeddings. Although deliberately simple, this design already learns a meaningful approximation as studied in the ablation studies. Two approximations in particular are worth revisiting. First, the predictor uses only visual embeddings: incorporating the question text could, in principle, sharpen difficulty estimates, but doing so via later LLM representations would require running part of the LM prefill before $K$ is decided, eroding the compute savings that motivate our pre-prefill design. Second, AVIS treats visual pruning and reasoning allocation as largely decoupled decisions, even though the retained visual context can directly affect both sample difficulty and the marginal value of additional reasoning. Future work can explore unified, text-aware, and budget-aware policies that jointly select visual context and reasoning effort, directly optimize the compute-constrained objective, and adapt to distribution shifts in deployment workloads.

AVIS is a first instantiation of adaptive VCS--VRS allocation. Predicting per-sample difficulty is itself a challenging problem, which we do not claim to solve fully. Instead, following recent adaptive-compute work for LLMs~\cite{snell2024scaling-llm,damani2024learning}, we approximate difficulty using a coarse binary solvability signal and a lightweight head over visual embeddings. Although deliberately simple, this design already learns a meaningful allocation signal, as shown in our ablation studies. Two approximations are particularly worth revisiting. First, the predictor uses only visual embeddings: incorporating the question text could sharpen difficulty estimates, but using later LLM representations would require running part of the LM prefill before deciding $K$, eroding the compute savings that motivate our pre-prefill design. Second, AVIS treats visual pruning and reasoning allocation as largely decoupled decisions, even though the retained visual context can directly affect both sample difficulty and the marginal value of additional reasoning. Future work can explore unified, text-aware, and budget-aware policies that jointly select visual context and reasoning effort, directly optimize the compute-constrained objective, and adapt to distribution shifts in deployment workloads.

% \clearpage

{
    \small
    \bibliographystyle{unsrt}
    \bibliography{main}

\begin{thebibliography}{10}

\bibitem{bai2025qwen2}
Shuai Bai, Keqin Chen, Xuejing Liu, Jialin Wang, Wenbin Ge, Sibo Song, Kai Dang, Peng Wang, Shijie Wang, Jun Tang, Humen Zhong, Yuanzhi Zhu, Mingkun Yang, Zhaohai Li, Jianqiang Wan, Pengfei Wang, Wei Ding, Zheren Fu, Yiheng Xu, Jiabo Ye, Xi~Zhang, Tianbao Xie, Zesen Cheng, Hang Zhang, Zhibo Yang, Haiyang Xu, and Junyang Lin.
\newblock {Qwen2.5-VL Technical Report}.
\newblock {\em arXiv preprint arXiv:2502.13923}, 2025.

\bibitem{li2024llava}
Bo~Li, Yuanhan Zhang, Dong Guo, Renrui Zhang, Feng Li, Hao Zhang, Kaichen Zhang, Yanwei Li, Ziwei Liu, and Chunyuan Li.
\newblock {LLaVA-OneVision: Easy Visual Task Transfer}.
\newblock {\em Transactions on Machine Learning Research}, 2025.

\bibitem{zhu2025internvl3}
Jinguo Zhu, Weiyun Wang, Zhe Chen, Zhaoyang Liu, Shenglong Ye, Lixin Gu, Hao Tian, Yuchen Duan, Weijie Su, Jie Shao, Zhangwei Gao, Erfei Cui, Xuehui Wang, Yue Cao, Yangzhou Liu, Xingguang Wei, Hongjie Zhang, Haomin Wang, Weiye Xu, Hao Li, Jiahao Wang, Nianchen Deng, Songze Li, Yinan He, Tan Jiang, Jiapeng Luo, Yi~Wang, Conghui He, Botian Shi, Xingcheng Zhang, Wenqi Shao, Junjun He, Yingtong Xiong, Wenwen Qu, Peng Sun, Penglong Jiao, Han Lv, Lijun Wu, Kaipeng Zhang, Huipeng Deng, Jiaye Ge, Kai Chen, Limin Wang, Min Dou, Lewei Lu, Xizhou Zhu, Tong Lu, Dahua Lin, Yu~Qiao, Jifeng Dai, and Wenhai Wang.
\newblock {InternVL3: Exploring Advanced Training and Test-Time Recipes for Open-Source Multimodal Models}.
\newblock {\em arXiv preprint arXiv:2504.10479}, 2025.

\bibitem{guo2025deepseek}
Daya Guo, Dejian Yang, Haowei Zhang, Junxiao Song, Peiyi Wang, Qihao Zhu, Runxin Xu, Ruoyu Zhang, Shirong Ma, Xiao Bi, Xiaokang Zhang, Xingkai Yu, Yu~Wu, Z.~F. Wu, Zhibin Gou, Zhihong Shao, Zhuoshu Li, Ziyi Gao, Aixin Liu, Bing Xue, Bingxuan Wang, Bochao Wu, Bei Feng, Chengda Lu, Chenggang Zhao, Chengqi Deng, Chong Ruan, Damai Dai, Deli Chen, Dongjie Ji, Erhang Li, Fangyun Lin, Fucong Dai, Fuli Luo, Guangbo Hao, Guanting Chen, Guowei Li, H.~Zhang, Hanwei Xu, Honghui Ding, Huazuo Gao, Hui Qu, Hui Li, Jianzhong Guo, Jiashi Li, Jingchang Chen, Jingyang Yuan, Jinhao Tu, Junjie Qiu, Junlong Li, J.~L. Cai, Jiaqi Ni, Jian Liang, Jin Chen, Kai Dong, Kai Hu, Kaichao You, Kaige Gao, Kang Guan, Kexin Huang, Kuai Yu, Lean Wang, Lecong Zhang, Liang Zhao, Litong Wang, Liyue Zhang, Lei Xu, Leyi Xia, Mingchuan Zhang, Minghua Zhang, Minghui Tang, Mingxu Zhou, Meng Li, Miaojun Wang, Mingming Li, Ning Tian, Panpan Huang, Peng Zhang, Qiancheng Wang, Qinyu Chen, Qiushi Du, Ruiqi Ge, Ruisong Zhang, Ruizhe Pan, Runji Wang, R.~J.
  Chen, R.~L. Jin, Ruyi Chen, Shanghao Lu, Shangyan Zhou, Shanhuang Chen, Shengfeng Ye, Shiyu Wang, Shuiping Yu, Shunfeng Zhou, Shuting Pan, S.~S. Li, Shuang Zhou, Shaoqing Wu, Tao Yun, Tian Pei, Tianyu Sun, T.~Wang, Wangding Zeng, Wen Liu, Wenfeng Liang, Wenjun Gao, Wenqin Yu, Wentao Zhang, W.~L. Xiao, Wei An, Xiaodong Liu, Xiaohan Wang, Xiaokang Chen, Xiaotao Nie, Xin Cheng, Xin Liu, Xin Xie, Xingchao Liu, Xinyu Yang, Xinyuan Li, Xuecheng Su, Xuheng Lin, X.~Q. Li, Xiangyue Jin, Xiaojin Shen, Xiaosha Chen, Xiaowen Sun, Xiaoxiang Wang, Xinnan Song, Xinyi Zhou, Xianzu Wang, Xinxia Shan, Y.~K. Li, Y.~Q. Wang, Y.~X. Wei, Yang Zhang, Yanhong Xu, Yao Li, Yao Zhao, Yaofeng Sun, Yaohui Wang, Yi~Yu, Yichao Zhang, Yifan Shi, Yiliang Xiong, Ying He, Yishi Piao, Yisong Wang, Yixuan Tan, Yiyang Ma, Yiyuan Liu, Yongqiang Guo, Yuan Ou, Yuduan Wang, Yue Gong, Yuheng Zou, Yujia He, Yunfan Xiong, Yuxiang Luo, Yuxiang You, Yuxuan Liu, Yuyang Zhou, Y.~X. Zhu, Yanping Huang, Yaohui Li, Yi~Zheng, Yuchen Zhu, Yunxian Ma, Ying
  Tang, Yukun Zha, Yuting Yan, Z.~Z. Ren, Zehui Ren, Zhangli Sha, Zhe Fu, Zhean Xu, Zhenda Xie, Zhengyan Zhang, Zhewen Hao, Zhicheng Ma, Zhigang Yan, Zhiyu Wu, Zihui Gu, Zijia Zhu, Zijun Liu, Zilin Li, Ziwei Xie, Ziyang Song, Zizheng Pan, Zhen Huang, Zhipeng Xu, Zhongyu Zhang, and Zhen Zhang.
\newblock {DeepSeek-R1 Incentivizes Reasoning in LLMs through Reinforcement Learning}.
\newblock {\em Nature}, 645:633–638, 2025.

\bibitem{huang2025vision-r1}
Wenxuan Huang, Bohan Jia, Zijie Zhai, Shaosheng Cao, Zheyu Ye, Fei Zhao, Zhe Xu, Yao Hu, and Shaohui Lin.
\newblock {Vision-R1: Incentivizing Reasoning Capability in Multimodal Large Language Models}.
\newblock In {\em International Conference on Learning Representations (ICLR)}, 2026.

\bibitem{wei2022chain}
Jason Wei, Xuezhi Wang, Dale Schuurmans, Maarten Bosma, Brian Ichter, Fei Xia, Ed~Chi, Quoc Le, and Denny Zhou.
\newblock {Chain-of-Thought Prompting Elicits Reasoning in Large Language Models}.
\newblock In {\em Advances in Neural Information Processing Systems (NeurIPS)}, 2022.

\bibitem{stiennon2020learning}
Nisan Stiennon, Long Ouyang, Jeff Wu, Daniel~M. Ziegler, Ryan Lowe, Chelsea Voss, Alec Radford, Dario Amodei, and Paul Christiano.
\newblock {Learning to Summarize with Human Feedback}.
\newblock In {\em Advances in Neural Information Processing Systems (NeurIPS)}, 2020.

\bibitem{ichihara2025evaluation}
Yuki Ichihara, Yuu Jinnai, Tetsuro Morimura, Kaito Ariu, Kenshi Abe, Mitsuki Sakamoto, and Eiji Uchibe.
\newblock {Evaluation of Best-of-N Sampling Strategies for Language Model Alignment}.
\newblock {\em Transactions on Machine Learning Research}, 2025.

\bibitem{wang2022self}
Xuezhi Wang, Jason Wei, Dale Schuurmans, Quoc Le, Ed~Chi, Sharan Narang, Aakanksha Chowdhery, and Denny Zhou.
\newblock {Self-Consistency Improves Chain of Thought Reasoning in Language Models}.
\newblock In {\em International Conference on Learning Representations (ICLR)}, 2023.

\bibitem{chen2024image}
Liang Chen, Haozhe Zhao, Tianyu Liu, Shuai Bai, Junyang Lin, Chang Zhou, and Baobao Chang.
\newblock {An Image is Worth 1/2 Tokens After Layer 2: Plug-and-Play Inference Acceleration for Large Vision-Language Models}.
\newblock In {\em European Conference on Computer Vision}, pages 19--35. Springer, 2024.

\bibitem{shang2025llava}
Yuzhang Shang, Mu~Cai, Bingxin Xu, Yong~Jae Lee, and Yan Yan.
\newblock {LLaVA-PruMerge: Adaptive Token Reduction for Efficient Large Multimodal Models}.
\newblock In {\em Proceedings of the IEEE/CVF International Conference on Computer Vision}, pages 22857--22867, 2025.

\bibitem{ahmadpour2025limits}
Mohammadjavad Ahmadpour, Amirmahdi Meighani, Payam Taebi, Omid Ghahroodi, Amirmohammad Izadi, and Mahdieh~Soleymani Baghshah.
\newblock {Limits and Gains of Test-Time Scaling in Vision-Language Reasoning}.
\newblock {\em arXiv preprint arXiv:2512.11109}, 2025.

\bibitem{liu2023visual}
Haotian Liu, Chunyuan Li, Qingyang Wu, and Yong~Jae Lee.
\newblock {Visual Instruction Tuning}.
\newblock In {\em Advances in Neural Information Processing Systems (NeurIPS)}, 2023.

\bibitem{li2023blip}
Junnan Li, Dongxu Li, Silvio Savarese, and Steven Hoi.
\newblock {BLIP-2: Bootstrapping Language-Image Pre-training with Frozen Image Encoders and Large Language Models}.
\newblock In {\em {International Conference on Machine Learning}}, 2023.

\bibitem{liu2024llavanext}
Haotian Liu, Chunyuan Li, Yuheng Li, Bo~Li, Yuanhan Zhang, Sheng Shen, and Yong~Jae Lee.
\newblock {LLaVA-NeXT: Improved Reasoning, OCR, and World Knowledge}, January 2024.

\bibitem{xu2025llava}
Guowei Xu, Peng Jin, Ziang Wu, Hao Li, Yibing Song, Lichao Sun, and Li~Yuan.
\newblock {LLaVA-CoT: Let Vision Language Models Reason Step-by-Step}.
\newblock In {\em Proceedings of the IEEE/CVF International Conference on Computer Vision}, pages 2087--2098, 2025.

\bibitem{yao2024mulberry}
Huanjin Yao, Jiaxing Huang, Wenhao Wu, Jingyi Zhang, Yibo Wang, Shunyu Liu, Yingjie Wang, Yuxin Song, Haocheng Feng, Li~Shen, and Dacheng Tao.
\newblock {Mulberry: Empowering MLLM with o1-like Reasoning and Reflection via Collective Monte Carlo Tree Search}.
\newblock In {\em Advances in Neural Information Processing Systems (NeurIPS)}, 2025.

\bibitem{wang2025vl}
Haozhe Wang, Chao Qu, Zuming Huang, Wei Chu, Fangzhen Lin, and Wenhu Chen.
\newblock {VL-Rethinker: Incentivizing Self-Reflection of Vision-Language Models with Reinforcement Learning}.
\newblock In {\em Advances in Neural Information Processing Systems (NeurIPS)}, 2025.

\bibitem{jeddi2025puzzle}
Ahmadreza Jeddi, Hakki~Can Karaimer, Hue Nguyen, Zhongling Wang, Ke~Zhao, Javad Rajabi, Ran Zhang, Raghav Goyal, Babak Taati, and Radek Grzeszczuk.
\newblock {Puzzle Curriculum GRPO for Vision-Centric Reasoning}.
\newblock {\em arXiv preprint arXiv:2512.14944}, 2025.

\bibitem{shen2025vlm}
Haozhan Shen, Peng Liu, Jingcheng Li, Chunxin Fang, Yibo Ma, Jiajia Liao, Qiaoli Shen, Zilun Zhang, Kangjia Zhao, Qianqian Zhang, Ruochen Xu, and Tiancheng Zhao.
\newblock {VLM-R1: A Stable and Generalizable R1-style Large Vision-Language Model}.
\newblock {\em arXiv preprint arXiv:2504.07615}, 2025.

\bibitem{zhang2025r1}
Jingyi Zhang, Jiaxing Huang, Huanjin Yao, Shunyu Liu, Xikun Zhang, Shijian Lu, and Dacheng Tao.
\newblock {R1-VL: Learning to Reason with Multimodal Large Language Models via Step-wise Group Relative Policy Optimization}.
\newblock In {\em Proceedings of the IEEE/CVF International Conference on Computer Vision (ICCV)}, 2025.

\bibitem{jeddi2026does}
Ahmadreza Jeddi, Kimia Shaban, Negin Baghbanzadeh, Natasha Sharan, Abhishek Moturu, Elham Dolatabadi, and Babak Taati.
\newblock {When Does RL Help Medical VLMs? Disentangling Vision, SFT, and RL Gains}.
\newblock {\em arXiv preprint arXiv:2603.01301}, 2026.

\bibitem{chen2021evaluating}
Mark Chen, Jerry Tworek, Heewoo Jun, Qiming Yuan, Henrique~Ponde de~Oliveira~Pinto, Jared Kaplan, Harri Edwards, Yuri Burda, Nicholas Joseph, Greg Brockman, Alex Ray, Raul Puri, Gretchen Krueger, Michael Petrov, Heidy Khlaaf, Girish Sastry, Pamela Mishkin, Brooke Chan, Scott Gray, Nick Ryder, Mikhail Pavlov, Alethea Power, Lukasz Kaiser, Mohammad Bavarian, Clemens Winter, Philippe Tillet, Felipe~Petroski Such, Dave Cummings, Matthias Plappert, Fotios Chantzis, Elizabeth Barnes, Ariel Herbert-Voss, William~Hebgen Guss, Alex Nichol, Alex Paino, Nikolas Tezak, Jie Tang, Igor Babuschkin, Suchir Balaji, Shantanu Jain, William Saunders, Christopher Hesse, Andrew~N. Carr, Jan Leike, Josh Achiam, Vedant Misra, Evan Morikawa, Alec Radford, Matthew Knight, Miles Brundage, Mira Murati, Katie Mayer, Peter Welinder, Bob McGrew, Dario Amodei, Sam McCandlish, Ilya Sutskever, and Wojciech Zaremba.
\newblock {Evaluating Large Language Models Trained on Code}.
\newblock {\em arXiv preprint arXiv:2107.03374}, 2021.

\bibitem{brown2024large}
Bradley Brown, Jordan Juravsky, Ryan Ehrlich, Ronald Clark, Quoc~V Le, Christopher R{\'e}, and Azalia Mirhoseini.
\newblock {Large Language Monkeys: Scaling Inference Compute with Repeated Sampling}.
\newblock {\em arXiv preprint arXiv:2407.21787}, 2024.

\bibitem{rakhsha2025majority}
Amin Rakhsha, Kanika Madan, Tianyu Zhang, Amir massoud Farahmand, and Amir Khasahmadi.
\newblock {Majority of the Bests: Improving Best-of-N via Bootstrapping}.
\newblock In {\em Advances in Neural Information Processing Systems (NeurIPS)}, 2025.

\bibitem{madaan2023self}
Aman Madaan, Niket Tandon, Prakhar Gupta, Skyler Hallinan, Luyu Gao, Sarah Wiegreffe, Uri Alon, Nouha Dziri, Shrimai Prabhumoye, Yiming Yang, Shashank Gupta, Bodhisattwa~Prasad Majumder, Katherine Hermann, Sean Welleck, Amir Yazdanbakhsh, and Peter Clark.
\newblock {Self-Refine: Iterative Refinement with Self-Feedback}.
\newblock In {\em Advances in Neural Information Processing Systems (NeurIPS)}, 2023.

\bibitem{lightman2023let}
Hunter Lightman, Vineet Kosaraju, Yuri Burda, Harrison Edwards, Bowen Baker, Teddy Lee, Jan Leike, John Schulman, Ilya Sutskever, and Karl Cobbe.
\newblock {Let's Verify Step by Step}.
\newblock In {\em The Twelfth International Conference on Learning Representations}, 2023.

\bibitem{wang2024math}
Peiyi Wang, Lei Li, Zhihong Shao, Runxin Xu, Damai Dai, Yifei Li, Deli Chen, Yu~Wu, and Zhifang Sui.
\newblock {Math-Shepherd: Verify and Reinforce LLMs Step-by-step without Human Annotations}.
\newblock In {\em Proceedings of the 62nd Annual Meeting of the Association for Computational Linguistics (Volume 1: Long Papers)}, pages 9426--9439, 2024.

\bibitem{muennighoff2025s1}
Niklas Muennighoff, Zitong Yang, Weijia Shi, Xiang~Lisa Li, Li~Fei-Fei, Hannaneh Hajishirzi, Luke Zettlemoyer, Percy Liang, Emmanuel Cand{\`e}s, and Tatsunori~B Hashimoto.
\newblock {s1: Simple Test-Time Scaling}.
\newblock In {\em Conference on Empirical Methods in Natural Language Processing (EMNLP)}, 2025.

\bibitem{li2025selfbudgeter}
Zheng Li, Qingxiu Dong, Jingyuan Ma, Di~Zhang, Kai Jia, and Zhifang Sui.
\newblock {SelfBudgeter: Adaptive Token Allocation for Efficient LLM Reasoning}.
\newblock {\em arXiv preprint arXiv:2505.11274}, 2025.

\bibitem{manvi2024adaptive}
Rohin Manvi, Anikait Singh, and Stefano Ermon.
\newblock {Adaptive Inference-Time Compute: LLMs Can Predict if They Can Do Better, Even Mid-Generation}.
\newblock {\em arXiv preprint arXiv:2410.02725}, 2024.

\bibitem{snell2024scaling-llm}
Charlie Snell, Jaehoon Lee, Kelvin Xu, and Aviral Kumar.
\newblock {Scaling LLM Test-Time Compute Optimally can be More Effective than Scaling Model Parameters}.
\newblock In {\em International Conference on Learning Representations (ICLR)}, 2025.

\bibitem{damani2024learning}
Mehul Damani, Idan Shenfeld, Andi Peng, Andreea Bobu, and Jacob Andreas.
\newblock {Learning How Hard to Think: Input-Adaptive Allocation of LM Computation}.
\newblock In {\em International Conference on Learning Representations (ICLR)}, 2025.

\bibitem{meng2025cyberv}
Jiahao Meng, Shuyang Sun, Yue Tan, Lu~Qi, Yunhai Tong, Xiangtai Li, and Longyin Wen.
\newblock {CyberV: Cybernetics for Test-time Scaling in Video Understanding}.
\newblock {\em arXiv preprint arXiv:2506.07971}, 2025.

\bibitem{jin2025cotvid}
Hongbo Jin, Jiayu Ding, Siyi Xie, Guibo Luo, and Ge~Li.
\newblock {VISTA: Mitigating Semantic Inertia in Video-LLMs via Training-Free Dynamic Chain-of-Thought Routing}.
\newblock {\em arXiv preprint arXiv:2505.11830}, 2025.

\bibitem{huang2026learning}
Yixu Huang, Tinghui Zhu, and Muhao Chen.
\newblock {Learning Adaptive Reasoning Paths for Efficient Visual Reasoning}.
\newblock {\em arXiv preprint arXiv:2604.14568}, 2026.

\bibitem{chen2025ares}
Shuang Chen, Yue Guo, Yimeng Ye, Shijue Huang, Wenbo Hu, Haoxi Li, Manyuan Zhang, Jiayu Chen, Song Guo, and Nanyun Peng.
\newblock {ARES: Multimodal Adaptive Reasoning via Difficulty-Aware Token-Level Entropy Shaping}.
\newblock In {\em International Conference on Learning Representations (ICLR)}, 2026.

\bibitem{kaya2025efficient}
Mehmet~Onurcan Kaya, Desmond Elliott, and Dim~P Papadopoulos.
\newblock {Efficient Test-Time Scaling for Small Vision-Language Models}.
\newblock In {\em International Conference on Learning Representations (ICLR)}, 2026.

\bibitem{zhang2025improve}
Ruohong Zhang, Bowen Zhang, Yanghao Li, Haotian Zhang, Zhiqing Sun, Zhe Gan, Yinfei Yang, Ruoming Pang, and Yiming Yang.
\newblock {Improve Vision Language Model Chain-of-thought Reasoning}.
\newblock In {\em Proceedings of the 63rd Annual Meeting of the Association for Computational Linguistics (Volume 1: Long Papers)}, pages 1631--1662, 2025.

\bibitem{wu2025aha}
Mingyuan Wu, Meitang Li, Jingcheng Yang, Jize Jiang, Kaizhuo Yan, Zhaoheng Li, Hanchao Yu, Minjia Zhang, and Klara Nahrstedt.
\newblock {Aha Moment Revisited: Are VLMs Truly Capable of Self Verification in Inference-time Scaling?}
\newblock {\em arXiv preprint arXiv:2506.17417}, 2025.

\bibitem{jeddi2025similarity}
Ahmadreza Jeddi, Negin Baghbanzadeh, Elham Dolatabadi, and Babak Taati.
\newblock {Similarity-Aware Token Pruning: Your VLM but Faster}.
\newblock {\em arXiv preprint arXiv:2503.11549}, 2025.

\bibitem{xing2024pyramiddrop}
Long Xing, Qidong Huang, Xiaoyi Dong, Jiajie Lu, Pan Zhang, Yuhang Zang, Yuhang Cao, Conghui He, Jiaqi Wang, Feng Wu, and Dahua Lin.
\newblock {PyramidDrop: Accelerating Your Large Vision-Language Models via Pyramid Visual Redundancy Reduction}.
\newblock In {\em Proceedings of the IEEE/CVF International Conference on Computer Vision}, 2025.

\bibitem{zhang2024sparsevlm}
Yuan Zhang, Chun-Kai Fan, Junpeng Ma, Wenzhao Zheng, Tao Huang, Kuan Cheng, Denis Gudovskiy, Tomoyuki Okuno, Yohei Nakata, Kurt Keutzer, and Shanghang Zhang.
\newblock {SparseVLM: Visual Token Sparsification for Efficient Vision-Language Model Inference}.
\newblock In {\em {International Conference on Machine Learning}}, 2025.

\bibitem{yang2025visionzip}
Senqiao Yang, Yukang Chen, Zhuotao Tian, Chengyao Wang, Jingyao Li, Bei Yu, and Jiaya Jia.
\newblock {VisionZip: Longer is Better but Not Necessary in Vision Language Models}.
\newblock In {\em Proceedings of the Computer Vision and Pattern Recognition Conference}, pages 19792--19802, 2025.

\bibitem{huang2025prunevid}
Xiaohu Huang, Hao Zhou, and Kai Han.
\newblock {PruneVid: Visual Token Pruning for Efficient Video Large Language Models}.
\newblock In {\em Findings of the Association for Computational Linguistics: ACL 2025}, pages 19959--19973, 2025.

\bibitem{bolya2022token}
Daniel Bolya, Cheng-Yang Fu, Xiaoliang Dai, Peizhao Zhang, Christoph Feichtenhofer, and Judy Hoffman.
\newblock {Token Merging: Your ViT But Faster}.
\newblock In {\em International Conference on Learning Representations (ICLR)}, 2023.

\bibitem{alvar2025divprune}
Saeed~Ranjbar Alvar, Gursimran Singh, Mohammad Akbari, and Yong Zhang.
\newblock {DivPrune: Diversity-based Visual Token Pruning for Large Multimodal Models}.
\newblock In {\em Proceedings of the Computer Vision and Pattern Recognition Conference}, pages 9392--9401, 2025.

\bibitem{zhang2025beyond-attention}
Qizhe Zhang, Mengzhen Liu, Lichen Li, Ming Lu, Yuan Zhang, Junwen Pan, Qi~She, and Shanghang Zhang.
\newblock {Beyond Attention or Similarity: Maximizing Conditional Diversity for Token Pruning in MLLMs}.
\newblock In {\em Advances in Neural Information Processing Systems (NeurIPS)}, 2025.

\bibitem{zhang2025beyond-text}
Qizhe Zhang, Aosong Cheng, Ming Lu, Renrui Zhang, Zhiyong Zhuo, Jiajun Cao, Shaobo Guo, Qi~She, and Shanghang Zhang.
\newblock {Beyond Text-Visual Attention: Exploiting Visual Cues for Effective Token Pruning in VLMs}.
\newblock In {\em Proceedings of the IEEE/CVF International Conference on Computer Vision}, 2025.

\bibitem{yang2025topv}
Cheng Yang, Yang Sui, Jinqi Xiao, Lingyi Huang, Yu~Gong, Chendi Li, Jinghua Yan, Yu~Bai, Ponnuswamy Sadayappan, Xia Hu, and Bo~Yuan.
\newblock {TopV: Compatible Token Pruning with Inference Time Optimization for Fast and Low-Memory Multimodal Vision Language Model}.
\newblock In {\em Proceedings of the Computer Vision and Pattern Recognition Conference}, 2025.

\bibitem{li2025catp}
Yanshu Li, Jianjiang Yang, Zhennan Shen, Ligong Han, Haoyan Xu, and Ruixiang Tang.
\newblock {CATP: Contextually Adaptive Token Pruning for Efficient and Enhanced Multimodal In-Context Learning}.
\newblock In {\em Proceedings of the AAAI Conference on Artificial Intelligence}, 2026.

\bibitem{dao2022flashattention}
Tri Dao, Dan Fu, Stefano Ermon, Atri Rudra, and Christopher R{\'e}.
\newblock {FlashAttention: Fast and Memory-Efficient Exact Attention with IO-Awareness}.
\newblock In {\em Advances in Neural Information Processing Systems (NeurIPS)}, 2022.

\bibitem{zhang2025vscan}
Ce~Zhang, Kaixin Ma, Tianqing Fang, Wenhao Yu, Hongming Zhang, Zhisong Zhang, Yaqi Xie, Katia Sycara, Haitao Mi, and Dong Yu.
\newblock {VScan: Rethinking Visual Token Reduction for Efficient Large Vision-Language Models}.
\newblock {\em Transactions on Machine Learning Research}, 2026.

\bibitem{park2025keydiff}
Junyoung Park, Dalton Jones, Matthew~J Morse, Raghavv Goel, Mingu Lee, and Chris Lott.
\newblock {KeyDiff: Key Similarity-Based KV Cache Eviction for Long-Context LLM Inference in Resource-Constrained Environments}.
\newblock In {\em Advances in Neural Information Processing Systems (NeurIPS)}, 2025.

\bibitem{duan2024vlmevalkit}
Haodong Duan, Xinyu Fang, Junming Yang, Xiangyu Zhao, Yuxuan Qiao, Mo~Li, Amit Agarwal, Zhe Chen, Lin Chen, Yuan Liu, Yubo Ma, Hailong Sun, Yifan Zhang, Shiyin Lu, Tack~Hwa Wong, Weiyun Wang, Peiheng Zhou, Xiaozhe Li, Chaoyou Fu, Junbo Cui, Jixuan Chen, Enxin Song, Song Mao, Shengyuan Ding, Tianhao Liang, Zicheng Zhang, Xiaoyi Dong, Yuhang Zang, Pan Zhang, Jiaqi Wang, Dahua Lin, and Kai Chen.
\newblock {VLMEvalKit: An Open-Source Toolkit for Evaluating Large Multi-Modality Models}.
\newblock In {\em Proceedings of the ACM International Conference on Multimedia}, 2024.

\bibitem{lu2023mathvista}
Pan Lu, Hritik Bansal, Tony Xia, Jiacheng Liu, Chunyuan Li, Hannaneh Hajishirzi, Hao Cheng, Kai-Wei Chang, Michel Galley, and Jianfeng Gao.
\newblock {MathVista: Evaluating Mathematical Reasoning of Foundation Models in Visual Contexts}.
\newblock In {\em International Conference on Learning Representations (ICLR)}, 2024.

\bibitem{zhang2024mathverse}
Renrui Zhang, Dongzhi Jiang, Yichi Zhang, Haokun Lin, Ziyu Guo, Pengshuo Qiu, Aojun Zhou, Pan Lu, Kai-Wei Chang, Peng Gao, and Hongsheng Li.
\newblock {MathVerse: Does Your Multi-modal LLM Truly See the Diagrams in Visual Math Problems?}
\newblock In {\em European Conference on Computer Vision}, 2024.

\bibitem{wang2024measuring}
Ke~Wang, Junting Pan, Weikang Shi, Zimu Lu, Houxing Ren, Aojun Zhou, Mingjie Zhan, and Hongsheng Li.
\newblock {Measuring Multimodal Mathematical Reasoning with MATH-Vision Dataset}.
\newblock {\em Advances in Neural Information Processing Systems (NeurIPS)}, 37:95095--95169, 2024.

\bibitem{mathew2021docvqa}
Minesh Mathew, Dimosthenis Karatzas, and CV~Jawahar.
\newblock {DocVQA: A Dataset for VQA on Document Images}.
\newblock In {\em Proceedings of the IEEE/CVF Winter Conference on Applications of Computer Vision (WACV)}, 2021.

\bibitem{yue2025mmmu}
Xiang Yue, Tianyu Zheng, Yuansheng Ni, Yubo Wang, Kai Zhang, Shengbang Tong, Yuxuan Sun, Botao Yu, Ge~Zhang, Huan Sun, Yu~Su, Wenhu Chen, and Graham Neubig.
\newblock {MMMU-Pro: A More Robust Multi-discipline Multimodal Understanding Benchmark}.
\newblock In {\em Proceedings of the 63rd Annual Meeting of the Association for Computational Linguistics (Volume 1: Long Papers)}, pages 15134--15186, 2025.

\bibitem{fu2025mme}
Chaoyou Fu, Peixian Chen, Yunhang Shen, Yulei Qin, Mengdan Zhang, Xu~Lin, Jinrui Yang, Xiawu Zheng, Ke~Li, Xing Sun, Yunsheng Wu, Rongrong Ji, Caifeng Shan, and Ran He.
\newblock {MME: A Comprehensive Evaluation Benchmark for Multimodal Large Language Models}.
\newblock In {\em Advances in Neural Information Processing Systems (NeurIPS)}, 2025.

\bibitem{chen2024we}
Lin Chen, Jinsong Li, Xiaoyi Dong, Pan Zhang, Yuhang Zang, Zehui Chen, Haodong Duan, Jiaqi Wang, Yu~Qiao, Dahua Lin, and Feng Zhao.
\newblock {Are We on the Right Way for Evaluating Large Vision-Language Models?}
\newblock In {\em Advances in Neural Information Processing Systems (NeurIPS)}, 2024.

\bibitem{liu2024mmbench}
Yuan Liu, Haodong Duan, Yuanhan Zhang, Bo~Li, Songyang Zhang, Wangbo Zhao, Yike Yuan, Jiaqi Wang, Conghui He, Ziwei Liu, Kai Chen, and Dahua Lin.
\newblock {MMBench: Is Your Multi-modal Model an All-around Player?}
\newblock In {\em European Conference on Computer Vision}, 2024.

\bibitem{tong2024cambrian}
Shengbang Tong, Ellis Brown, Penghao Wu, Sanghyun Woo, Manoj Middepogu, Sai~Charitha Akula, Jihan Yang, Shusheng Yang, Adithya Iyer, Xichen Pan, Ziteng Wang, Rob Fergus, Yann LeCun, and Saining Xie.
\newblock {Cambrian-1: A Fully Open, Vision-Centric Exploration of Multimodal LLMs}.
\newblock In {\em Advances in Neural Information Processing Systems (NeurIPS)}, 2024.

\bibitem{li2023evaluating}
Yifan Li, Yifan Du, Kun Zhou, Jinpeng Wang, Wayne~Xin Zhao, and Ji-Rong Wen.
\newblock {Evaluating Object Hallucination in Large Vision-Language Models}.
\newblock In {\em Conference on Empirical Methods in Natural Language Processing (EMNLP)}, 2023.

\bibitem{fu2024blink}
Xingyu Fu, Yushi Hu, Bangzheng Li, Yu~Feng, Haoyu Wang, Xudong Lin, Dan Roth, Noah~A. Smith, Wei-Chiu Ma, and Ranjay Krishna.
\newblock {BLINK: Multimodal Large Language Models Can See but Not Perceive}.
\newblock In {\em European Conference on Computer Vision}, 2024.

\bibitem{wang2026treebench}
Haochen Wang, Xiangtai Li, Zilong Huang, Anran Wang, Jiacong Wang, Tao Zhang, Jiani Zheng, Sule Bai, Zijian Kang, Jiashi Feng, Zhuochen Wang, and Zhaoxiang Zhang.
\newblock {Traceable Evidence Enhanced Visual Grounded Reasoning: Evaluation and Methodology}.
\newblock In {\em International Conference on Learning Representations (ICLR)}, 2026.

\bibitem{fu2025video}
Chaoyou Fu, Yuhan Dai, Yongdong Luo, Lei Li, Shuhuai Ren, Renrui Zhang, Zihan Wang, Chenyu Zhou, Yunhang Shen, Mengdan Zhang, Peixian Chen, Yanwei Li, Shaohui Lin, Sirui Zhao, Ke~Li, Tong Xu, Xiawu Zheng, Enhong Chen, Caifeng Shan, Ran He, and Xing Sun.
\newblock {Video-MME: The First-Ever Comprehensive Evaluation Benchmark of Multi-modal LLMs in Video Analysis}.
\newblock In {\em Proceedings of the Computer Vision and Pattern Recognition Conference}, 2025.

\bibitem{liu2024tempcompass}
Yuanxin Liu, Shicheng Li, Yi~Liu, Yuxiang Wang, Shuhuai Ren, Lei Li, Sishuo Chen, Xu~Sun, and Lu~Hou.
\newblock {TempCompass: Do Video LLMs Really Understand Videos?}
\newblock In {\em Findings of the Association for Computational Linguistics: ACL 2024}, pages 8731--8772, 2024.

\bibitem{zhang2025towards}
Yuanhan Zhang, Yunice Chew, Yuhao Dong, Aria Leo, Bo~Hu, and Ziwei Liu.
\newblock {Towards Video Thinking Test: A Holistic Benchmark for Advanced Video Reasoning and Understanding}.
\newblock In {\em Proceedings of the IEEE/CVF International Conference on Computer Vision}, 2025.

\bibitem{li2024mvbench}
Kunchang Li, Yali Wang, Yinan He, Yizhuo Li, Yi~Wang, Yi~Liu, Zun Wang, Jilan Xu, Guo Chen, Ping Luo, Limin Wang, and Yu~Qiao.
\newblock {MVBench: A Comprehensive Multi-modal Video Understanding Benchmark}.
\newblock In {\em Proceedings of the IEEE/CVF Conference on Computer Vision and Pattern Recognition}, 2024.

\bibitem{zhang2025q}
Zicheng Zhang, Ziheng Jia, Haoning Wu, Chunyi Li, Zijian Chen, Yingjie Zhou, Wei Sun, Xiaohong Liu, Xiongkuo Min, Weisi Lin, and Guangtao Zhai.
\newblock {Q-Bench-Video: Benchmark the Video Quality Understanding of LMMs}.
\newblock In {\em Proceedings of the Computer Vision and Pattern Recognition Conference}, 2025.

\bibitem{hu2025video}
Kairui Hu, Penghao Wu, Fanyi Pu, Wang Xiao, Yuanhan Zhang, Xiang Yue, Bo~Li, and Ziwei Liu.
\newblock {Video-MMMU: Evaluating Knowledge Acquisition from Multi-Discipline Professional Videos}.
\newblock {\em arXiv preprint arXiv:2501.13826}, 2025.

\bibitem{kwon2023efficient}
Woosuk Kwon, Zhuohan Li, Siyuan Zhuang, Ying Sheng, Lianmin Zheng, Cody~Hao Yu, Joseph~E. Gonzalez, Hao Zhang, and Ion Stoica.
\newblock Efficient memory management for large language model serving with pagedattention.
\newblock In {\em Proceedings of the ACM SIGOPS 29th Symposium on Operating Systems Principles}, 2023.

\bibitem{deng2025openvl}
Yihe Deng, Hritik Bansal, Fan Yin, Nanyun Peng, Wei Wang, and Kai-Wei Chang.
\newblock {OpenVLThinker: Complex Vision-Language Reasoning via Iterative SFT-RL Cycles}.
\newblock In {\em Advances in Neural Information Processing Systems (NeurIPS)}, 2025.

\bibitem{jaegle2021perceiver}
Andrew Jaegle, Felix Gimeno, Andy Brock, Oriol Vinyals, Andrew Zisserman, and Joao Carreira.
\newblock Perceiver: General perception with iterative attention.
\newblock In {\em International conference on machine learning}, pages 4651--4664. PMLR, 2021.

\bibitem{zhang2024video}
Yuanhan Zhang, Jinming Wu, Wei Li, Bo~Li, Zejun Ma, Ziwei Liu, and Chunyuan Li.
\newblock {LLaVA-Video: Video Instruction Tuning With Synthetic Data}.
\newblock {\em Transactions on Machine Learning Research}, 2025.

\bibitem{schwenk2022okvqa}
Dustin Schwenk, Apoorv Khandelwal, Christopher Clark, Kenneth Marino, and Roozbeh Mottaghi.
\newblock {A-OKVQA: A Benchmark for Visual Question Answering using World Knowledge}.
\newblock In {\em European Conference on Computer Vision}, 2022.

\bibitem{kembhavi2016diagram}
Aniruddha Kembhavi, Mike Salvato, Eric Kolve, Minjoon Seo, Hannaneh Hajishirzi, and Ali Farhadi.
\newblock {A Diagram Is Worth A Dozen Images}.
\newblock In {\em European Conference on Computer Vision}, 2016.

\bibitem{kembhavi2017you}
Aniruddha Kembhavi, Minjoon Seo, Dustin Schwenk, Jonghyun Choi, Ali Farhadi, and Hannaneh Hajishirzi.
\newblock {Are You Smarter Than a Sixth Grader? Textbook Question Answering for Multimodal Machine Comprehension}.
\newblock In {\em Proceedings of the IEEE Conference on Computer Vision and Pattern Recognition}, 2017.

\bibitem{li2023seed}
Bohao Li, Rui Wang, Guangzhi Wang, Yuying Ge, Yixiao Ge, and Ying Shan.
\newblock {SEED-Bench: Benchmarking Multimodal LLMs with Generative Comprehension}.
\newblock In {\em Proceedings of the IEEE/CVF Conference on Computer Vision and Pattern Recognition}, 2024.

\bibitem{lu2022learn}
Pan Lu, Swaroop Mishra, Tanglin Xia, Liang Qiu, Kai-Wei Chang, Song-Chun Zhu, Oyvind Tafjord, Peter Clark, and Ashwin Kalyan.
\newblock {Learn to Explain: Multimodal Reasoning via Thought Chains for Science Question Answering}.
\newblock {\em Advances in Neural Information Processing Systems}, 35:2507--2521, 2022.

\bibitem{masry2022chartqa}
Ahmed Masry, Xuan~Long Do, Jia~Qing Tan, Shafiq Joty, and Enamul Hoque.
\newblock {ChartQA: A Benchmark for Question Answering about Charts with Visual and Logical Reasoning}.
\newblock In {\em Findings of the Association for Computational Linguistics: ACL 2022}, 2022.

\bibitem{yue2024mmmu}
Xiang Yue, Yuansheng Ni, Kai Zhang, Tianyu Zheng, Ruoqi Liu, Ge~Zhang, Samuel Stevens, Dongfu Jiang, Weiming Ren, Yuxuan Sun, Cong Wei, Botao Yu, Ruibin Yuan, Renliang Sun, Ming Yin, Boyuan Zheng, Zhenzhu Yang, Yibo Liu, Wenhao Huang, Huan Sun, Yu~Su, and Wenhu Chen.
\newblock {MMMU: A Massive Multi-discipline Multimodal Understanding and Reasoning Benchmark for Expert AGI}.
\newblock In {\em Proceedings of the IEEE/CVF Conference on Computer Vision and Pattern Recognition}, 2024.

\bibitem{hudson2019gqa}
Drew~A Hudson and Christopher~D Manning.
\newblock {GQA: A New Dataset for Real-World Visual Reasoning and Compositional Question Answering}.
\newblock In {\em Proceedings of the IEEE/CVF Conference on Computer Vision and Pattern Recognition}, 2019.

\end{thebibliography}
}

% Hide main-paper sections from the appendix TOC
\addtocontents{toc}{\protect\setcounter{tocdepth}{-1}}

\newpage
\tableofcontents
\clearpage

% Add only the appendix heading and appendix sections to the TOC
\phantomsection
\addcontentsline{toc}{section}{Appendix}
\addtocontents{toc}{\protect\setcounter{tocdepth}{2}}

%-------------------------------------------------------------------------------------------------
\label{appendix}
\appendix
% \onecolumn
\section{Proofs}
\label{supp:sec:proofs}

\subsection{Token-linearity of Language Model FLOPs}
\begin{proposition}[Token-linearity of LM FLOPs]
\label{prop:constant}
Let $n_{\mathrm{ctx}}=n_v+n_q$ denote the multimodal prefix length (visual + text tokens) and let $T$ be the number of decoded tokens. For a given model $\mathcal{M}$, define $C_{\mathcal{M}}$ as the average FLOPs incurred by the LM stack per processed token (summed over all transformer layers, excluding the vision encoder). Then the total inference FLOPs for one completion satisfy
\begin{equation}
\label{eq:flops_pseudo_linear}
F(I,Q,T) \;\approx\; F_v(I) \;+\; C_{\mathcal{M}}\,(n_{\mathrm{ctx}} + T).
\end{equation}
\end{proposition}
\begin{proof}
We separate the vision encoder from the LM stack:
\begin{equation}
F(I,Q,T) \;=\; F_v(I) \;+\; F_{\mathrm{LM}}(n_{\mathrm{ctx}},T).
\label{eq:decomp_vision_lm}
\end{equation}
It suffices to show that $F_{\mathrm{LM}}(n_{\mathrm{ctx}},T)$ is well-approximated by a constant times $(n_{\mathrm{ctx}}+T)$.

We use the per-layer FLOPs model
\begin{equation}
C_{\mathrm{pre}}(n)=4 d\,n^2 + B_0 n,
\qquad
C_{\mathrm{step}}(n)=4 d\,n + B_0,
\label{eq:cost_model_linear_proof}
\end{equation}
where token-linear terms (projections + MLP) are grouped into $B_0$:
\begin{equation}
B_0 \;:=\; 4d^2 \;+\; 4d\,d_{\mathrm{kv}} \;+\; \gamma d m,
\label{eq:B0_linear_proof}
\end{equation}
In~\eqref{eq:cost_model_linear_proof}, $n$ is the sequence length (tokens), $C_{\mathrm{pre}}(n)$ is per-layer prefill FLOPs (KV-cache build), and $C_{\mathrm{step}}(n)$ is per-layer FLOPs for one KV-cached decode step at cache length $n$.
In~\eqref{eq:B0_linear_proof}, $d$ is the LLM hidden size, $d_{\mathrm{kv}}$ is the KV projection width, $m$ is the FFN/MLP intermediate size, $B_0$ groups token-linear terms (projections + FFN), and $\gamma$ is the FFN constant ($\gamma=6$ for SwiGLU).

Let the LM have $L$ layers. Under KV-cached decoding:
\begin{itemize}
    \item \textbf{Prefill:} build the KV cache on the prefix of length $n_{\mathrm{ctx}}$, costing $L\,C_{\mathrm{pre}}(n_{\mathrm{ctx}})$.
    \item \textbf{Decode:} at step $t\in\{1,\dots,T\}$, the cache length is $n_{\mathrm{ctx}}+(t-1)$, costing  $L\,C_{\mathrm{step}}(n_{\mathrm{ctx}}+t-1)$.
\end{itemize}
Therefore,
\begin{equation}
F_{\mathrm{LM}}(n_{\mathrm{ctx}},T)
=
L\,C_{\mathrm{pre}}(n_{\mathrm{ctx}})
+\sum_{t=1}^{T} L\,C_{\mathrm{step}}(n_{\mathrm{ctx}}+t-1).
\label{eq:lm_exact_start}
\end{equation}

Substitute~\eqref{eq:cost_model_linear_proof} into~\eqref{eq:lm_exact_start}:
\begin{align}
F_{\mathrm{LM}}(n_{\mathrm{ctx}},T)
&=
L\big(4 d\,n_{\mathrm{ctx}}^2 + B_0 n_{\mathrm{ctx}}\big)
\nonumber\\
&\quad+\sum_{t=1}^{T}L\big(4 d\,(n_{\mathrm{ctx}}+t-1)+B_0\big)
\nonumber\\
&=
L B_0(n_{\mathrm{ctx}}+T)
\nonumber\\
&\quad+4Ld\!\left(
n_{\mathrm{ctx}}^2
+\sum_{t=1}^{T}(n_{\mathrm{ctx}}+t-1)
\right).
\label{eq:lm_expand_mid}
\end{align}

Compute the sum:
\[
\sum_{t=1}^{T}(n_{\mathrm{ctx}}+t-1)
=
\sum_{t=1}^{T}n_{\mathrm{ctx}}+\sum_{t=1}^{T}(t-1)
=
n_{\mathrm{ctx}}T+\frac{T(T-1)}{2}.
\]
Plugging in gives the exact expression
\begin{equation}
\begin{aligned}
F_{\mathrm{LM}}(n_{\mathrm{ctx}},T)
&=
L B_0(n_{\mathrm{ctx}}+T)
\\
&\quad+
4Ld\!\left(
n_{\mathrm{ctx}}^2
+n_{\mathrm{ctx}}T
+\frac{T(T-1)}{2}
\right).
\end{aligned}
\label{eq:lm_exact_expand_final}
\end{equation}

\par\noindent\textbf{Binding nonlinear term by linear function of $(n_{\mathrm{ctx}}+T)$.}
Assume an operating regime with
\[
0 \le n_{\mathrm{ctx}}\le n_{\max},
\qquad
0 \le T\le T_{\max}.
\]
Define
\begin{equation}
\kappa \;:=\; n_{\max}+\frac{T_{\max}-1}{2}.
\label{eq:kappa_def}
\end{equation}
We bound the three nonlinear terms one-by-one:
\[
n_{\mathrm{ctx}}^2 \le n_{\max}n_{\mathrm{ctx}},
n_{\mathrm{ctx}}T \le n_{\max}T,
\frac{T(T-1)}{2} \le \frac{T_{\max}-1}{2}\,T.
\]
Adding these inequalities gives
{\small
\begin{align}
n_{\mathrm{ctx}}^2+n_{\mathrm{ctx}}T+\frac{T(T-1)}{2}
&\le
n_{\max}n_{\mathrm{ctx}}
+n_{\max}T
+\frac{T_{\max}-1}{2}\,T
\nonumber\\
&=
n_{\max}n_{\mathrm{ctx}}
+\Big(n_{\max}+\frac{T_{\max}-1}{2}\Big)T
\nonumber\\
&=
n_{\max}n_{\mathrm{ctx}}+\kappa T.
\label{eq:nonlinear_bound_intermediate}
\end{align}
}

Since $\kappa \ge n_{\max}$ and $n_{\mathrm{ctx}}\ge 0$, we have $n_{\max}n_{\mathrm{ctx}}\le \kappa n_{\mathrm{ctx}}$, hence
\begin{equation}
n_{\mathrm{ctx}}^2+n_{\mathrm{ctx}}T+\frac{T(T-1)}{2}
\;\le\;
\kappa n_{\mathrm{ctx}}+\kappa T
=
\kappa\,(n_{\mathrm{ctx}}+T).
\label{eq:nonlinear_linear_bound}
\end{equation}

\par\noindent\textbf{Linear sandwich and token-linearity.}
Start from~\eqref{eq:lm_exact_expand_final}.
The bracketed term is nonnegative, so
\begin{equation}
F_{\mathrm{LM}}(n_{\mathrm{ctx}},T)
\;\ge\;
L B_0(n_{\mathrm{ctx}}+T).
\label{eq:lm_lower_bound}
\end{equation}
Using~\eqref{eq:nonlinear_linear_bound} in~\eqref{eq:lm_exact_expand_final} gives
\begin{align}
F_{\mathrm{LM}}(n_{\mathrm{ctx}},T)
&\le
L B_0(n_{\mathrm{ctx}}+T)
+
4Ld\,\kappa\,(n_{\mathrm{ctx}}+T)
\nonumber\\
&=
L(B_0+4d\kappa)\,(n_{\mathrm{ctx}}+T).
\label{eq:lm_upper_bound}
\end{align}
Thus
\begin{equation}
L B_0(n_{\mathrm{ctx}}+T)
\;\le\;
F_{\mathrm{LM}}(n_{\mathrm{ctx}},T)
\;\le\;
L(B_0+4d\kappa)\,(n_{\mathrm{ctx}}+T).
\label{eq:linear_sandwich_final}
\end{equation}
We therefore use the token-linear surrogate
$F_{\mathrm{LM}}(n_{\mathrm{ctx}},T)\approx C_{\mathcal{M}}(n_{\mathrm{ctx}}+T)$,
where $C_{\mathcal{M}}$ is measured once for $\mathcal{M}$, and substituting into~\eqref{eq:decomp_vision_lm} yields~\eqref{eq:flops_pseudo_linear}.
\end{proof}

% \par\noindent\textbf{Numerical example: computing $C_{\mathcal{M}}$ (Qwen2.5-VL-7B).}
% We define the per-token LM-stack constant at an operating point $(n_{\mathrm{ctx}},T_0)$ as
% \[
% C_{\mathcal{M}}
% \;:=\;
% \frac{F_{\mathrm{LM}}(n_{\mathrm{ctx}},T_0)}{n_{\mathrm{ctx}}+T_0}.
% \]
% For Qwen2.5-VL-7B (Fig.~\textbf{X}), we use $L=28$, $d=3584$, $H_{\mathrm{kv}}=4$, $d_h=128$
% (so $d_{\mathrm{kv}}=H_{\mathrm{kv}}d_h=512$), $m=18944$, and SwiGLU ($\gamma=6$), hence
% \[
% B_0
% =4d^2+4d\,d_{\mathrm{kv}}+\gamma d m
% =466{,}092{,}032.
% \]
% Dividing the exact expression in~\eqref{eq:lm_exact_expand_final} by $(n_{\mathrm{ctx}}+T_0)$ gives
% \[
% C_{\mathcal{M}}
% =
% L B_0
% +
% \frac{4Ld}{n_{\mathrm{ctx}}+T_0}
% \left(
% n_{\mathrm{ctx}}^2+n_{\mathrm{ctx}}T_0+\frac{T_0(T_0-1)}{2}
% \right).
% \]
% At a representative point example of one data with $n_{\mathrm{ctx}}=1500$ and $T_0=200$ (so $n_{\mathrm{ctx}}+T_0=1700$),
% \[
% C_{\mathcal{M}}\;\approx\;
% 13.66\ \text{GFLOPs/token}.
% \]

% \caption{\textbf{Dataset-level token scaling for self-consistency.}
% Let $T_{i,k}$ be the number of decoded tokens for example $i$ and rollout $k$ under the same decoding policy.
% We empirically verify that the total number of generated tokens scales approximately linearly with $K$,
% i.e., $\sum_{i}\sum_{k=1}^K T_{i,k} \approx K \sum_{i} T_{i,1}$.
% By Proposition~\ref{prop:prune_generate_nq_nv} (token-linearity), this implies decoding FLOPs scale approximately linearly with $K$.}
% \label{tab:token_scaling_dataset}
% \end{table}

\subsection{FLOPs Scaling With $K$ Trajectories Under Shared-prefill}
\begin{theorem}[FLOPs scaling with $K$ trajectories under shared-prefill]
\label{thm:flops_scaling_vlm}
Under shared-prefill execution, all $K$ trajectories reuse the same vision features and prefill KV cache. If the $k$-th trajectory decodes $T_k$ tokens, then
\begin{equation}
\label{eq:k_sample_flops_vlm}
\small{
F_K(I,Q;T_{1:K})
=F_v(I)+F_{\mathrm{pre}}(n_v,n_q)+\sum_{k=1}^{K} F_{\mathrm{dec}}(n_v,n_q,T_k),
}
\end{equation}
where $F_K$ is the inference FLOPs for generating $Y_K$, or K trajectories. If $T_k \overset{\mathrm{i.i.d.}}{\sim} T$, then
$\mathbb{E}[F_K(I,Q)]
=
F_v(I)+F_{\mathrm{pre}}(n_v,n_q)+K\,\mathbb{E}[F_{\mathrm{dec}}(n_v,n_q,T)]$.
\end{theorem}
\label{sec:appendix_proofs}
\begin{proof}[Proof]
\label{proof:flops_scaling_vlm}
The vision encoder depends only on $I$ and is computed once, incurring $F_v(I)$ FLOPs; its output is reused across samples.
Given the vision features and prompt text $s$, the prefill pass that constructs the KV cache depends only on $(I,s)$ and is likewise computed once,
incurring $F_{\mathrm{pre}}(I,s)$ FLOPs.
For each sample $k$, autoregressive generation produces $T_k$ tokens; by definition, decoding $T_k$ tokens given the cached states costs
$F_{\mathrm{dec}}(I,s,T_k)$ FLOPs.
Decoding computations across different samples cannot be shared beyond the cached states because they depend on distinct sampled token sequences,
hence the total FLOPs equal the shared vision and prefill costs plus the sum of per-sample decoding costs, yielding \autoref{eq:k_sample_flops_vlm}.
Taking expectations over i.i.d.\ lengths $T_k$ gives \autoref{thm:flops_scaling_vlm}.
\end{proof}

% requires: \usepackage{booktabs} \usepackage{multirow}

% \begin{table*}[t]
% \centering
% \setlength{\tabcolsep}{5pt}
% \resizebox{\textwidth}{!}{%
% \begin{tabular}{c|ccc|ccc|c|c|c}
% \toprule
% $K$
% & $n_q$
% & $n_v$
% & $T_k$
% & $F_v(I)$ (T)
% & $F_{\mathrm{pre}}(n_v,n_q)$ (T)
% & $\sum_{k=1}^K F_{\mathrm{dec}}(n_v,n_q,T_k)$ (T)
% & FLOPs / gen.\ token ($C_{\mathcal{M}}$) (G)
% & $F_K(I,Q;T_{1:K})$ (T)
% & $\widehat{\mathbb{E}}[F_K(I,Q)]$ (T) \\
% \midrule
% \multicolumn{10}{l}{\small $\widehat{\mathbb{E}}[T]=85.1$ estimated from $K=10$ rollouts; \;
% $\widehat{\mathbb{E}}[F_K]=F_v+F_{\mathrm{pre}}+\frac{K\,C_{\mathcal{M}}}{1000}\widehat{\mathbb{E}}[T]$.} \\
% \midrule
% 1
% & 84 & 1302 & 76
% & 6.680 & 19.601 & 1.075
% & \multirow{3}{*}{14.142}
% & 27.356
% & 27.484 \\
% 3
% & 84 & 1302 & 76/104/94
% & 6.680 & 19.601 & 3.875
% &
% & 30.156
% & 29.891 \\
% 5
% & 84 & 1302 & 76/104/94/91/84
% & 6.680 & 19.601 & 6.350
% &
% & 32.631
% & 32.298 \\
% \bottomrule
% \end{tabular}%
% }
% \caption{\textbf{FLOPs additivity and expected scaling under shared-prefill for one sample.}
% FLOPs are reported in TFLOPs (T), except $C_{\mathcal{M}}$ in GFLOPs/token (G).
% The last column reports the expectation-style estimate
% $\widehat{\mathbb{E}}[F_K(I,Q)]$, where $\widehat{\mathbb{E}}[T]$ is estimated from $K{=}10$ rollouts (average generated length $85.1$ tokens).}
% \label{tab:flops_sanity}
% \end{table*}

\subsubsection{Empirical example of the FLOPs model}
% \autoref{tab:flops_sanity} provides an empirical sanity check for our FLOPs decomposition (Theorem~\ref{thm:flops_scaling_vlm}) and the token-linearity approximation (Proposition~\ref{prop:prune_generate_nq_nv}).
% Under shared-prefill execution, the vision and prefill costs ($F_v$ and $F_{\mathrm{pre}}$) are incurred once and remain constant as $K$ varies, while the total decoding cost grows with the number of generated tokens across trajectories.
% Consistent with Proposition~\ref{prop:prune_generate_nq_nv}, we observe a stable per-token decoding cost, reported as $C_{\mathcal{M}}$ (GFLOPs/token), which allows us to form an expectation-style estimate of total cost:
% $\widehat{\mathbb{E}}[F_K(I,Q)] = F_v(I)+F_{\mathrm{pre}}(n_v,n_q)+\frac{K\,C_{\mathcal{M}}}{1000}\widehat{\mathbb{E}}[T]$.
% The close agreement between this estimate and the measured totals across $K$ supports the additivity in~\eqref{eq:k_sample_flops_vlm} and justifies using $C_{\mathcal{M}}$ as a compact summary of decoding cost when analyzing the VCS--VRS trade-off.

Theorem~\ref{thm:flops_scaling_vlm} shows that under shared-prefill, increasing $K$ affects only the decoding term, which sums over trajectories.
If rollouts are sampled i.i.d.\ under a fixed decoding policy, the total number of generated tokens over a curated dataset scales approximately linearly with $K$:
$\sum_{i\in\mathcal{D}}\sum_{k=1}^{K} T_{i,k}\approx K\sum_{i\in\mathcal{D}}T_{i,1}$ (~\autoref{tab:scaling_linear}).
Combined with Proposition~\ref{prop:constant}, which models decoding FLOPs as approximately linear in decoded tokens, this yields near-linear scaling of decoding compute with $K$, while $F_v$ and $F_{\mathrm{pre}}$ remain one-time costs.

\begin{table}[t]
\centering
\caption{Token counts and scaling behavior as a function of $K$. The table shows that increasing $K$ induces an approximately linear growth in total decoded tokens, with an $8.7\times$ increase at $K=9$ relative to $K=1$, confirming that inference scaling is dominated by token accumulation across self-consistency passes.
}
\label{tab:scaling_linear}
\small
\setlength{\tabcolsep}{24pt}
\resizebox{1.0\linewidth}{!}{%
\begin{tabular}{c|c|c|c}
\toprule
$K$
& $\sum_{i\in\mathcal{D}} T_{i,1}$
& $\sum_{i\in\mathcal{D}}\sum_{k=1}^{K} T_{i,k}$
& Token ratio $\left(\frac{\sum_{i}\sum_{k\le K}T_{i,k}}{\sum_{i}T_{i,1}}\right)$ \\
\midrule
1  & \multirow{5}{*}{155,038} & 155,038   & $1.00$ \\
3  &                          & 459,114   & 2.96   \\
5  &                          & 753,010   & 4.86   \\
7  &                          & 1,055,468 & 6.81   \\
9  &                          & 1,352,350 & 8.72   \\
\bottomrule
\end{tabular}%
}
\end{table}

\section{Data Curation for Training the Difficulty Predictor}
\label{supp:sec:data_curation}

We curate a calibration set for training the difficulty predictor by downselecting from an initial pool of
multimodal samples (62{,}430 images and 10{,}709 videos) collected from open-source benchmarks, including
LLaVA-Video-178K~\cite{zhang2024video}, AOKVQA~\cite{schwenk2022okvqa}, AI2D~\cite{kembhavi2016diagram},
TQA~\cite{kembhavi2017you}, SEEDBench~\cite{li2023seed}, ScienceQA~\cite{lu2022learn}, ChartQA~\cite{masry2022chartqa},
and MMMU~\cite{yue2024mmmu}. Our filtering follows three principles, inspired by S1~\cite{muennighoff2025s1}:
\textbf{Quality}, \textbf{Difficulty}, and \textbf{Diversity}.

\subsection{Quality Filtering}
We use Qwen2.5-VL-32B as an automatic judge to remove low-quality or ill-posed items, including:
(i) vague or underspecified questions, (ii) poor formatting or ambiguous answer options, (iii) cases where the
image/video is not required to solve the question, and (iv) corrupt or low-quality visual inputs. We also
retain only multiple-choice questions with explicit answer options. After this stage, we retain 45{,}000 samples.

\subsection{Difficulty Estimation and Validity Checks}
We estimate difficulty using Qwen2.5-VL-7B with CoT prompting. For each candidate, we run $5$ stochastic passes (temperature sampling) and record the empirical success rate. Since all questions are multiple-choice, we apply
a lightweight answer extractor and use a Qwen-7B judge to verify that each output (i) conforms to the expected
choice format and (ii) matches the ground-truth answer. We discard samples for which any of the $5$ passes yields an invalidly formatted output. After this stage, 31{,}423 samples remain.

\subsection{Diversity}
To encourage broad coverage across visual domains and question types, we use Qwen2.5-VL-32B to assign each candidate
to one of $36$ coarse subject categories (e.g., documents/graphics, indoor scenes, outdoor scenes, people, sports,
tools, STEM, diagrams, charts). We then construct a subject-stratified subset by sampling approximately uniformly
across categories while preserving a balanced spread over the difficulty spectrum. This yields 8{,}000 samples.

\subsection{Balancing and De-duplication}
To avoid over-representing particular difficulty bands, we further subsample approximately uniformly over the
success-rate spectrum. Finally, because several benchmarks share overlapping visuals, we remove near-duplicates
between the training pool and evaluation benchmarks. Specifically, we embed all images using CLIP and discard any
training example whose visual embedding has cosine similarity greater than $0.95$ with an evaluation image. For videos,
we apply the same criterion using uniformly sampled frames and remove a video if any sampled frame exceeds this
threshold. This conservative threshold is intended to remove exact or near-exact visual overlap while avoiding the
removal of merely semantically similar examples. The final dataset contains 4{,}000 images and 1{,}000 videos.

\section{Ablation Studies}
\subsection{Adaptive Rollout Allocation}
\label{app:adaptive_rollout_allocation}
\autoref{fig:judged_k_barplots} shows the empirical distribution of rollout counts selected by the learned VRS controller, where $K \in \{1,3,5,7\}$. The allocation is highly dataset-dependent, indicating that the controller learns non-trivial compute profiles rather than collapsing to a fixed decoding budget. The average selected rollout count is $2.12$ on MMBench, $2.53$ on MME, and $3.79$ on Video-TT, reflecting increasing use of reasoning-time search on benchmarks where additional rollouts are more likely to be beneficial.

In addition, \autoref{fig:judged_k_barplots} shows that AVIS preserves low-cost inference for many queries while selectively increasing compute for a targeted subset. On MMBench, the controller is relatively conservative, but still assigns larger budgets to non-negligible fractions of the data: $14.2\%$ of samples receive $K=5$ and $7.6\%$ receive $K=7$. On MME, more than one-third of samples are assigned $K>1$, suggesting that the predictor identifies a substantial set of examples where self-consistency may improve reliability. The effect is strongest on Video-TT, where $K=5$ is the most frequent assignment, covering $51.4\%$ of samples, consistent with the greater reasoning and temporal grounding demands of video understanding.

These results confirm that the learned rollout selector performs input-adaptive reasoning allocation. It does not simply increase test-time compute uniformly, nor does it degenerate to greedy decoding. Instead, it amortizes the inference budget across the workload: easy or low-gain examples are handled with a small $K$, while additional rollouts are concentrated on examples predicted to benefit from reasoning-time search. This behavior explains how AVIS can capture part of the accuracy gain of self-consistency while maintaining a substantially lower average compute cost.

\subsection{Ablating the Learned Difficulty Predictor}
\label{app:difficulty_predictor_ablation}
\textbf{Effect of calibration set size.}
We first study how the size of the calibration pool affects predictor quality.
Using a fixed held-out split of 250 samples, predictor accuracy improves
monotonically as the amount of training data increases. Training with $20\%$,
$50\%$, $75\%$, and $100\%$ of the remaining calibration pool yields predictor
accuracies of $64.4\%$, $73.8\%$, $77.4\%$, and $79.2\%$, respectively.
This trend indicates that the solvability signal is learnable from a moderate
number of calibration examples, and that the full 5000 calibration pool provides
a stable operating point for adaptive rollout selection.

\textbf{Architecture and end-to-end impact.}
We compare our AVIS predictor against alternative lightweight designs. On the
held-out calibration split, a two-layer Transformer and a Perceiver-style
predictor achieve $72.0\%$ and $69.6\%$ predictor accuracy, respectively, while
our final predictor achieves the highest held-out accuracy, $79.2\%$. We then
evaluate these variants end-to-end on MathVista and MMBench, as shown in
\autoref{tab:predictor_arch_ablation}.

\begin{table}[t]
\centering
\caption{
\textbf{Ablation of difficulty predictor architectures.} We keep the pruning policy, decoding settings, and budget calibration fixed, and change only the predictor architecture. Avg. $K$ denotes the average number of selected reasoning rollouts, and \#F denotes normalized FLOPs relative to the vanilla $K=1$ baseline.
}
\label{tab:predictor_arch_ablation}
\setlength{\tabcolsep}{16pt}
\resizebox{1.0\linewidth}{!}{%
\begin{tabular}{lccccccc}
\toprule
\multirow{2}{*}{Predictor}
& \multirow{2}{*}{Pred. Acc.}
& \multicolumn{3}{c}{MathVista}
& \multicolumn{3}{c}{MMBench} \\
\cmidrule(lr){3-5}
\cmidrule(lr){6-8}
& & Score & Avg. $K$ & \#F
& Score & Avg. $K$ & \#F \\
\midrule
Perceiver        & 69.6 & 68.1 & 3.57 & 0.64 & 86.6 & 3.25 & 0.59 \\
Transformer      & 72.0 & 68.6 & 2.93 & 0.53 & 86.3 & 2.57 & 0.51 \\
\cellcolor{winblue}AVIS (Ours)             & \cellcolor{winblue}79.2 & \cellcolor{winblue}68.1 & \cellcolor{winblue}2.33 & \cellcolor{winblue}0.46 & \cellcolor{winblue}86.4 & \cellcolor{winblue}2.12 & \cellcolor{winblue}0.44 \\
\bottomrule
\end{tabular}
}
\end{table}

The end-to-end results show that predictor quality primarily improves compute
allocation rather than simply increasing downstream accuracy. The Perceiver and
Transformer variants select larger rollout counts, leading to higher FLOPs with
little or no accuracy gain. On MathVista, our predictor matches the Perceiver
score of $68.1$ while reducing average $K$ from $3.57$ to $2.33$ and FLOPs from
$0.64\times$ to $0.46\times$. The Transformer obtains a slightly higher
MathVista score of $68.6$, but requires $2.93$ average rollouts and
$0.53\times$ FLOPs. On MMBench, our predictor achieves comparable accuracy to
the Perceiver, $86.4$ versus $86.6$, while reducing average $K$ from $3.25$ to
$2.12$ and FLOPs from $0.59\times$ to $0.44\times$. Compared with the
Transformer, our predictor also uses fewer rollouts and lower FLOPs on both
benchmarks, with comparable downstream accuracy.

These results suggest that higher predictor accuracy translates into a better
accuracy--FLOPs trade-off. Less accurate predictors produce less reliable
solvability estimates, which causes the binning rule to assign larger $K$ more
often. This increases compute without consistent downstream gains. In contrast,
our AVIS predictor produces sharper and more useful solvability estimates,
avoiding unnecessary rollouts while preserving accuracy.

\textbf{Predictor overhead.}
The learned difficulty predictor itself adds negligible inference overhead. It
consists of a stack of 1D Conv--GroupNorm--SiLU layers followed by global
average pooling and a two-layer MLP. Since it operates on the visual token
embeddings already computed by the vision encoder, it is executed only once per
input and does not require any additional VLM forward pass. Its average cost is
approximately $6$ GFLOPs per example, corresponding to less than $0.1\%$ of
total inference FLOPs. This overhead is far smaller than the compute saved by
visual token pruning or by avoiding unnecessary reasoning rollouts.

\subsection{Ablations on KDV}
\label{supp:sec:kdv_pruning}

In this section, we evaluate \textbf{KDV} as a standalone visual token pruning method.
We use Qwen2.5-VL-7B-Instruct and report single-pass inference results ($K{=}1$), without enforcing the explicit \texttt{<think>} format used in our main experiments.\footnote{The model may still generate chain-of-thought text even without enforcing the format.}
All experiments follow the VLMEvalKit default image resolution ($1280$) and use a fixed pruning configuration that retains 320 out of 1280 visual tokens (approximately $75\%$ pruning).

\begin{figure*}[t]
    \centering
    \includegraphics[width=1.0\textwidth]{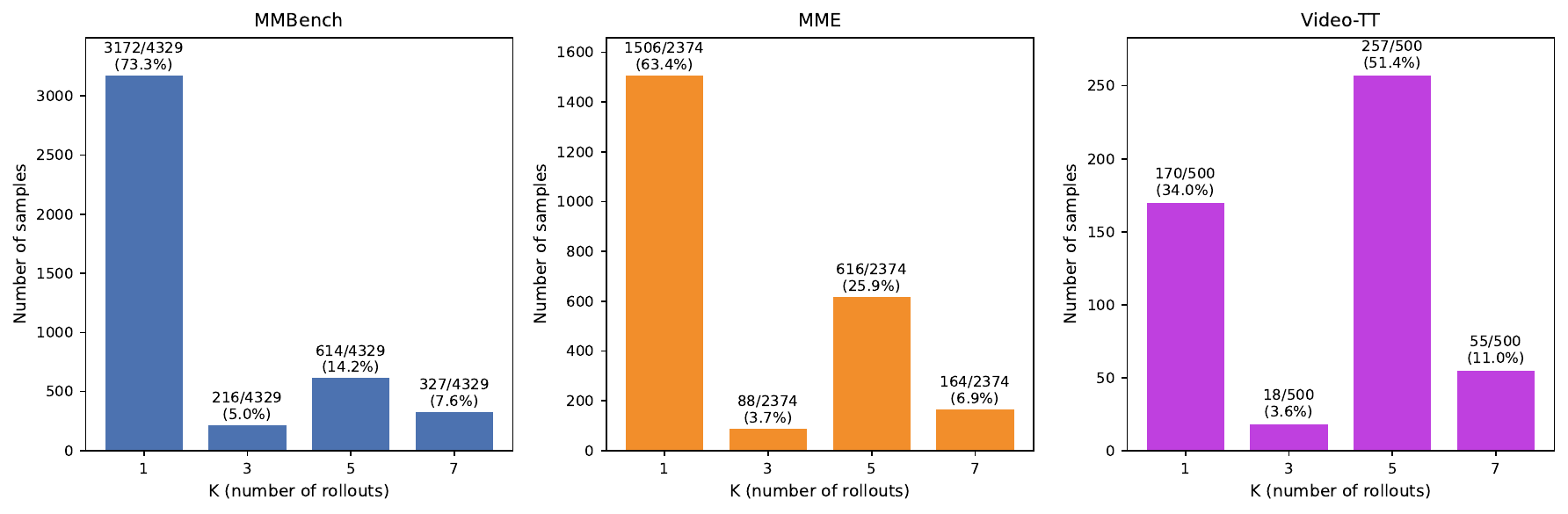}
\caption{\textbf{Distribution of predicted rollout counts $K$.}
Histogram of the learned difficulty-aware policy's $K$ assignments on representative image and video benchmarks. Because the binning rule reflects the inverted-U marginal utility of majority voting, $K{=}1$ is assigned at both the easy and very-hard extremes, while $K{=}5$ and $K{=}7$ are reserved for intermediate hard-but-solvable examples. The resulting distribution is dataset-dependent, shifting from $K{=}1$-dominated on MMBench to $K{=}5$-dominated on Video-TT.}
    % \vspace{-1em}
    \label{fig:judged_k_barplots}
\end{figure*}

\autoref{tab:kdv_pruning_main} summarizes performance on GQA, MMBench-DEV-EN, MME (perception+reasoning sum), POPE, ScienceQA-Image, CV-Bench-2D, SeedBench-Image, and MMStar.
Compared to the full-token \emph{Vanilla} baseline, all training-free pruning methods incur only modest degradation when retaining 25\% of visual tokens.
FastV achieves on average $95.48\%$ of the baseline score, VisionZip $97.86\%$, PDrop $97.45\%$, and KeyDiff $97.55\%$.
KDV attains the best average retention at \textbf{97.99\%}, with small gains on token-heavy benchmarks such as MME (101.13\% of Vanilla) and competitive performance elsewhere.
These results indicate that the key-diversity-based criterion used by KDV is a strong drop-in alternative to existing training-free pruning strategies when operating under tight visual token budgets.
\begin{table*}[t]
\centering
\small
\caption{\textbf{KDV pruning on Qwen-2.5-VL-Instruct.}
We report absolute scores and percentages relative to the full-token (\emph{Vanilla}) baseline. For this ablation, we do not enforce the \texttt{<think>} prompting format. Results are shown on GQA~\cite{hudson2019gqa}, MMBench-DEV-EN~\cite{liu2024mmbench}, MME~\cite{fu2025mme} (perception+reasoning sum), POPE~\cite{li2023evaluating} (overall accuracy), ScienceQA-Image~\cite{lu2022learn}, CV-Bench-2D~\cite{tong2024cambrian}, SeedBench-Image~\cite{li2023seed}, and MMStar~\cite{chen2024we}. We compare against Vanilla and representative training-free pruning baselines: FastV~\cite{chen2024image} ($K{=}2$), VisionZip~\cite{yang2025visionzip} (dominant tokens $=22\%$, context $=3\%$), PDrop~\cite{xing2024pyramiddrop} ($L{=}[8,16,24]$, $R{=}0.5$), and KeyDiff~\cite{park2025keydiff}.}

\setlength{\tabcolsep}{5pt}
\resizebox{\textwidth}{!}{%
\begin{tabular}{l|cccccccc|c}
\toprule
\textbf{Method}
& \textbf{GQA}
& \textbf{MMB}
& \textbf{MME}
& \textbf{POPE}
& \textbf{SQA}
& \textbf{CVBench}
& \textbf{SEED}
& \textbf{MMStar}
& \textbf{Avg.} \\
\midrule

\rowcolor{lightgray}
\multicolumn{10}{c}{\textit{Upper Bound, 1280 Tokens (100\%)}} \\
\midrule

\multirow{2}{*}{Vanilla}
& 60.26 & 78.69 & 2299 & 87.87 & 72.83  & 74.83 & 76.91 & 60.33
& \multirow{2}{*}{\centering 100\%} \\
& 100\% & 100\% & 100\% & 100\% & 100\% & 100\% & 100\% & 100\% & \\
\midrule

\rowcolor{lightgray}
\multicolumn{10}{c}{\textit{Retain 320 Tokens ($\downarrow$ 75\%)}} \\
\midrule

\multirow{2}{*}{FastV}
& 56.74 & 77.66 & 2300 & 85.71 & 71.54 & 70.04 & 70.04 & 54.6
& \multirow{2}{*}{\centering 95.48\%} \\
& 94.16\% & 98.69\% & 100.04\% & 97.54\% & 98.23\% & 93.60\% & 91.07\% & 90.50\% & \\
\midrule

\multirow{2}{*}{VisionZip}
& 59.07 & 77.92 & 2300 & 87.06 & 73.07 & 73.3 & 73.76 & 55.8
& \multirow{2}{*}{\centering 97.86\%} \\
& 98.03\% & 99.02\% & 100.04\% & 99.08\% & 100.33\% & 97.96\% & 95.90\% & 92.49\% & \\
\midrule

\multirow{2}{*}{PDrop}
& 56.91 & 77.41 & 2284 & 87.57 & 72.33 & 74.02 & 74.48 & 55.93
& \multirow{2}{*}{\centering 97.45\%} \\
& 94.44\% & 98.37\% & 99.35\% & 99.66\% & 99.31\% & 98.92\% & 96.84\% & 92.71\% & \\
\midrule

\multirow{2}{*}{KeyDiff}
& 58.03 & 77.21 & 2309 & 86.87 & 73.08 & 73.3 & 73.78 & 55.8
& \multirow{2}{*}{\centering 97.55\%} \\
& 96.30\% & 98.12\% & 100.43\% & 98.86\% & 100.34\% & 97.96\% & 95.93\% & 92.49\% & \\
\midrule

\rowcolor{winblue}
\textbf{KDV (Ours) }
& 58.76 & 78.53 & 2325 & 86.87 & 73.38 & 73.14 & 73.44 & 55.87 &  \textbf{97.99\%} \\
\rowcolor{winblue}
& 97.51\% & 99.80\% & 101.13\% & 98.86\% & 100.76\% & 97.74\% & 95.49\% & 92.61\% &\\
\bottomrule
\end{tabular}
}

\label{tab:kdv_pruning_main}
\end{table*}

%%%%%

\begin{figure}[t]
    \centering
    \includegraphics[width=0.75\textwidth]{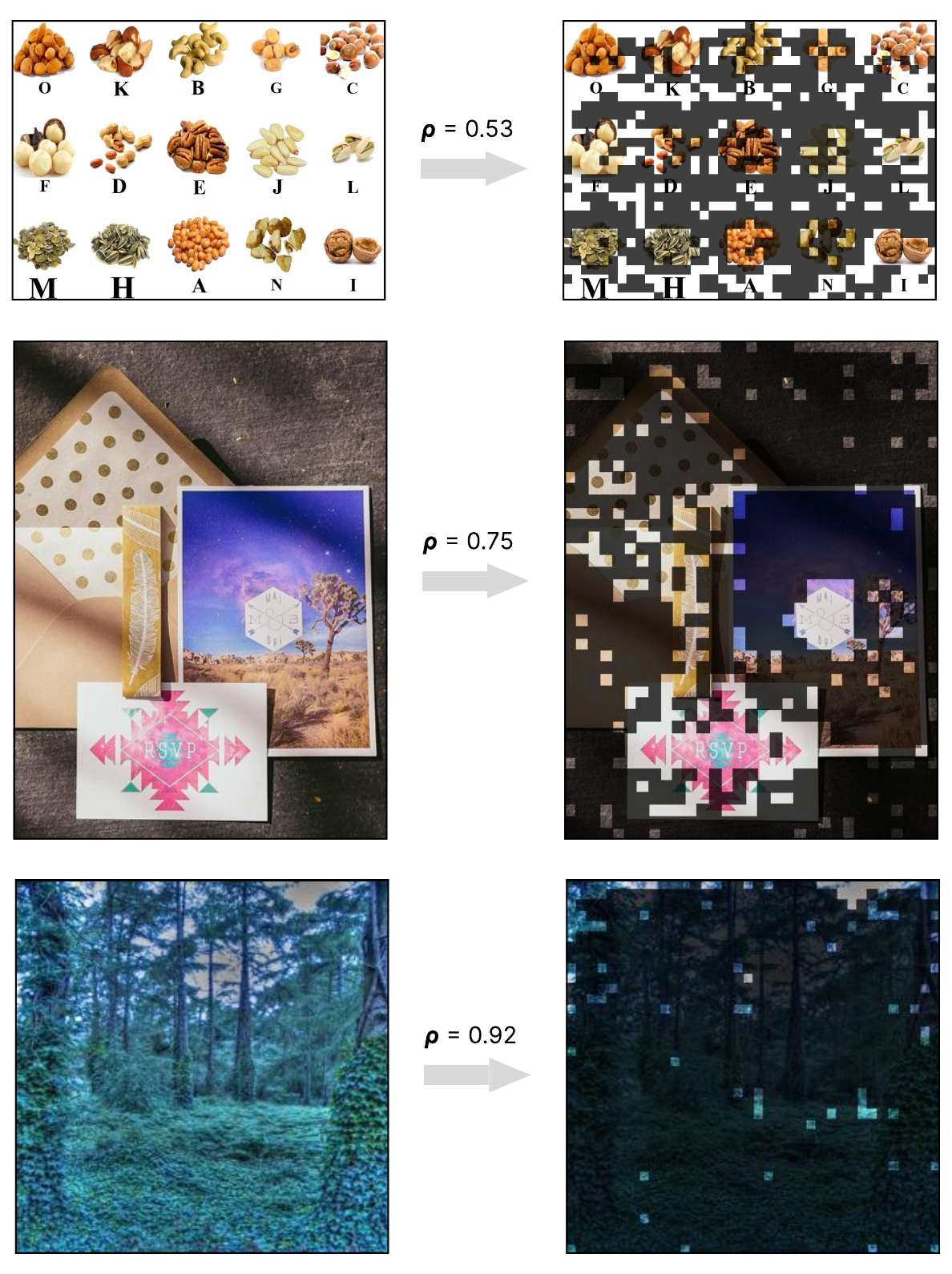}
    \caption{\textbf{Adaptive KDV pruning.}
    Qualitative examples of our adaptive KDV policy selecting different pruning ratios $\rho$ per image.
    In each pair, the left panel shows the original image and the right panel shows the retained visual tokens after pruning.\textit{ Images are from "Are We on the Right Way for Evaluating Large Vision-Language Models?" by Chen et al., published in NeurIPS 2024.}}
    \label{fig:adaptive_prune}
\end{figure}

\section{Limitations and Future Work}
\label{appendix:limitations}

AVIS is intentionally a lightweight, deployment-friendly instantiation of adaptive VCS--VRS allocation, and several of its design choices open natural directions for future work.

\noindent\textbf{Joint VCS--VRS policies.}
AVIS factorizes the inference policy into a vision-side pruning step (KDV) and a difficulty-aware rollout selector. This factorization is what makes both decisions runnable before LM prefill and is central to the compute savings we report. A natural extension is a single policy that selects $\rho$ and $K$ jointly under an explicit per-query compute constraint; in principle, such a policy could exploit interactions between visual context and reasoning effort that our staged design does not model. Our preliminary explorations suggest, however, that effective joint policies are non-trivial to design: naive instantiations we tried did not outperform AVIS, indicating that this direction likely requires richer training signals or differentiable surrogates for the discrete pruning and rollout decisions. We view this as an important and open problem rather than a straightforward extension.

\noindent\textbf{Text-aware difficulty prediction.}
Our difficulty predictor uses only the visual embeddings produced by the vision encoder. This choice is what allows $K$ to be selected before any LM forward pass and is the reason adaptive rollout selection composes cleanly with shared-prefill inference. Incorporating the question text could in principle sharpen difficulty estimates, but the most natural way to do so, via later LM representations~\cite{damani2024learning}, would require running part of the LM prefill before $K$ is decided, partially eroding these savings. Lightweight text-aware predictors that operate before the main LM prefill (e.g., over a small auxiliary text encoder, or over very early LM layers) are a promising direction that preserves the pre-prefill property AVIS relies on.

\noindent\textbf{Composition with other test-time scaling axes.}
We instantiate VRS through self-consistency with majority voting because it composes naturally with shared-prefill inference and is well-supported across model families. Other test-time scaling axes, sequential refinement, longer chain-of-thought, tree search, and Best-of-$N$ with learned verifiers or process reward models, are orthogonal to our allocation framework rather than competing with it, and combining adaptive parallel rollouts with adaptive sequential reasoning is a natural next step.

\section{Societal Impact}
\label{appendix:societal-impact}

AVIS is a test-time inference policy for existing vision-language models; it does not introduce new model capabilities, training data, or model weights. Its primary effect is to reduce the compute, latency, and energy required for VLM inference at comparable or improved accuracy, which can lower the cost and carbon footprint of multimodal model deployment. The qualitative capabilities and failure modes of the underlying VLMs, including hallucination, biased outputs, and potential for misuse such as automated surveillance or multimodal disinformation, are inherited from the base model and should be addressed with the same safeguards as for the underlying VLM.

\clearpage
%-------------------------------------------------------------------------------------------------
\addtocontents{toc}{\protect\setcounter{tocdepth}{-1}}
% NEURIPS PAPER CHECKLIST
\section*{NeurIPS Paper Checklist}

\begin{enumerate}

\item {\bf Claims}
    \item[] Question: Do the main claims made in the abstract and introduction accurately reflect the paper's contributions and scope?
    \item[] Answer: \answerYes{} % Replace by \answerYes{}, \answerNo{}, or \answerNA{}.
    \item[] Justification: The abstract and Section 1 clearly outline our core contributions: characterizing the trade-offs of Visual Context Scaling (VCS) and Visual Reasoning Scaling (VRS), and introducing the AVIS policy and Key Diversity Visual (KDV) pruning rule.

\item {\bf Limitations}
    \item[] Question: Does the paper discuss the limitations of the work performed by the authors?
    \item[] Answer: \answerYes{} % Replace by \answerYes{}, \answerNo{}, or \answerNA{}.
    \item[] Justification: We discuss the limitations of our decoupled policy design for visual pruning and reasoning scaling in Section 6 (Conclusion) and suggest directions for future unified approaches.

\item {\bf Theory assumptions and proofs}
    \item[] Question: For each theoretical result, does the paper provide the full set of assumptions and a complete (and correct) proof?
    \item[] Answer: \answerYes{} % Replace by \answerYes{}, \answerNo{}, or \answerNA{}.
    \item[] Justification: We present our FLOPs model and theoretical estimations in Section 3, with complete derivations, formal proofs, and assumptions fully detailed in Appendix A.

    \item {\bf Experimental result reproducibility}
    \item[] Question: Does the paper fully disclose all the information needed to reproduce the main experimental results of the paper to the extent that it affects the main claims and/or conclusions of the paper (regardless of whether the code and data are provided or not)?
    \item[] Answer: \answerYes{} % Replace by \answerYes{}, \answerNo{}, or \answerNA{}.
    \item[] Justification: We thoroughly describe our baseline configurations, dataset curation process, hyperparameter settings (e.g., pruning threshold $\tau=\pi/4$), and evaluation frameworks in Section 5.1 and Appendix B.

\item {\bf Open access to data and code}
    \item[] Question: Does the paper provide open access to the data and code, with sufficient instructions to faithfully reproduce the main experimental results, as described in supplemental material?
    \item[] Answer: \answerNo{} % Replace by \answerYes{}, \answerNo{}, or \answerNA{}.
    \item[] Justification: The code and full calibration dataset are not included in the current anonymous submission to preserve double-blind review, but will be open-sourced upon acceptance.

\item {\bf Experimental setting/details}
    \item[] Question: Does the paper specify all the training and test details (e.g., data splits, hyperparameters, how they were chosen, type of optimizer) necessary to understand the results?
    \item[] Answer: \answerYes{} % Replace by \answerYes{}, \answerNo{}, or \answerNA{}.
    \item[] Justification: We detail the evaluation frameworks, baseline settings, temperature and top-p sampling hyper-parameters in Section 5.1, as well as the calibration dataset generation process in Appendix B.

\item {\bf Experiment statistical significance}
    \item[] Question: Does the paper report error bars suitably and correctly defined or other appropriate information about the statistical significance of the experiments?
    \item[] Answer: \answerNo{}  % Replace by \answerYes{}, \answerNo{}, or \answerNA{}.
    \item[] Justification: Error bars are not reported due to the prohibitively high computational cost of running multiple random seeds for large vision-language models evaluated across dozens of extensive image and video benchmarks. 

\item {\bf Experiments compute resources}
    \item[] Question: For each experiment, does the paper provide sufficient information on the computer resources (type of compute workers, memory, time of execution) needed to reproduce the experiments?
    \item[] Answer: \answerYes{} % Replace by \answerYes{}, \answerNo{}, or \answerNA{}.
    \item[] Justification: Hardware details (a single NVIDIA L40 48GB GPU) and specific execution frameworks (vLLM v0.10.0) are documented alongside exact wall-clock latency measurements in Section 5.3.
    
\item {\bf Code of ethics}
    \item[] Question: Does the research conducted in the paper conform, in every respect, with the NeurIPS Code of Ethics \url{https://neurips.cc/public/EthicsGuidelines}?
    \item[] Answer: \answerYes{}  % Replace by \answerYes{}, \answerNo{}, or \answerNA{}.
    \item[] Justification: The research utilizes open-source models and benchmark datasets for algorithmic efficiency analysis, conforming to the Code of Ethics.

\item {\bf Broader impacts}
    \item[] Question: Does the paper discuss both potential positive societal impacts and negative societal impacts of the work performed?
    \item[] Answer: \answerNo{} % Replace by \answerYes{}, \answerNo{}, or \answerNA{}.
    \item[] Justification: This paper focuses on foundational algorithm design (inference efficiency). The societal impacts of our method align entirely with those of general VLM development, thus we do not include a dedicated broader impact section.
    
\item {\bf Safeguards}
    \item[] Question: Does the paper describe safeguards that have been put in place for responsible release of data or models that have a high risk for misuse (e.g., pre-trained language models, image generators, or scraped datasets)?
    \item[] Answer: \answerNA{} % Replace by \answerYes{}, \answerNo{}, or \answerNA{}.
    \item[] Justification: Our work proposes an inference-time acceleration technique and does not release any new high-risk pre-trained models or scraped datasets.

\item {\bf Licenses for existing assets}
    \item[] Question: Are the creators or original owners of assets (e.g., code, data, models), used in the paper, properly credited and are the license and terms of use explicitly mentioned and properly respected?
    \item[] Answer: \answerYes{} % Replace by \answerYes{}, \answerNo{}, or \answerNA{}.
    \item[] Justification: Existing models (Qwen2.5-VL series) and standard evaluation benchmarks are accurately referenced throughout Section 5 and Appendix B.

\item {\bf New assets}
    \item[] Question: Are new assets introduced in the paper well documented and is the documentation provided alongside the assets?
    \item[] Answer: \answerNA{} % Replace by \answerYes{}, \answerNo{}, or \answerNA{}.
    \item[] Justification: We do not introduce newly collected datasets or foundational model architectures; our work evaluates an inference strategy over existing open-source benchmarks.

\item {\bf Crowdsourcing and research with human subjects}
    \item[] Question: For crowdsourcing experiments and research with human subjects, does the paper include the full text of instructions given to participants and screenshots, if applicable, as well as details about compensation (if any)? 
    \item[] Answer: \answerNA{} % Replace by \answerYes{}, \answerNo{}, or \answerNA{}.
    \item[] Justification: The research does not involve crowdsourcing or human subjects.

\item {\bf Institutional review board (IRB) approvals or equivalent for research with human subjects}
    \item[] Question: Does the paper describe potential risks incurred by study participants, whether such risks were disclosed to the subjects, and whether Institutional Review Board (IRB) approvals (or an equivalent approval/review based on the requirements of your country or institution) were obtained?
    \item[] Answer: \answerNA{} % Replace by \answerYes{}, \answerNo{}, or \answerNA{}.
    \item[] Justification: The research does not involve human subjects, rendering IRB approval inapplicable.

\item {\bf Declaration of LLM usage}
    \item[] Question: Does the paper describe the usage of LLMs if it is an important, original, or non-standard component of the core methods in this research? Note that if the LLM is used only for writing, editing, or formatting purposes and does \emph{not} impact the core methodology, scientific rigor, or originality of the research, declaration is not required.
    %this research? 
    \item[] Answer: \answerYes{} % Replace by \answerYes{}, \answerNo{}, or \answerNA{}.
    \item[] Justification: We explicitly document our use of VLM/LLM models (e.g., Qwen2.5-VL-32B and Qwen2.5-VL-7B) as automatic judges to filter data and estimate difficulty during the curation of our calibration dataset, as outlined in Appendix B.

\end{enumerate}

\end{document}